\documentclass[11pt]{article}
\usepackage[T1]{fontenc}
\usepackage[margin=1in]{geometry}   

\usepackage{graphicx}               
\usepackage{booktabs}               
\usepackage{amsmath,amsfonts}       
\usepackage{hyperref}               
\usepackage{algorithm,algorithmic}  
\usepackage{xcolor}

\newtheorem{theorem}{Theorem}
\newtheorem{proposition}[theorem]{Proposition} 
\newtheorem{definition}[theorem]{Definition} 

\newenvironment{proof}{\noindent\textbf{Proof.}}{\hfill$\square$\medskip}

\begin{document}
	
	\title{Convex--Concave Quadratic Spectral Filtering for Graph Neural Networks\thanks{Accepted at ECML-PKDD 2026. The final authenticated publication will be available online via Springer LNCS. The work described in this paper was supported partially by the National Natural Science Foundation of China (12271111).}}
	
	\author{
		Ranhui Yan\textsuperscript{1}, 
		Jia Cai\textsuperscript{2}\thanks{The corresponding author is Jia Cai.},
		Mengzhu Chen\textsuperscript{2}, 
		Haodong Yang\textsuperscript{2} \\[2ex]
		\small \textsuperscript{1}School of Computing and Artificial Intelligence, Guangzhou Xinhua University,\\
		\small Guangzhou, Guangdong, China\\
		\small \texttt{yan\_rh1999@outlook.com}\\[2ex]
		\small \textsuperscript{2}School of Statistics and Data Science, Guangdong University of Finance and Economics,\\
		\small Guangzhou, Guangdong, China\\
		\small \texttt{jiacai1999@gdufe.edu.cn}, \texttt{mengzhuchen@student.gdufe.edu.cn}, \texttt{haodongyang@student.gdufe.edu.cn}
	}
	\date{}
\maketitle
\begin{abstract}
	Spectral graph neural networks (GNNs) interpret message passing as frequency-selective filtering. While low-order spectral filters are efficient, their limited selectivity often leads to weak attenuation outside the passband, whereas high-order alternatives introduce optimization challenges. We propose DCQ-GNN, a spectral GNN based on a compact bank of adaptive convex--concave quadratic filters. By restricting the filter order to two while explicitly exploiting complementary curvature, DCQ-GNN improves spectral selectivity as quantified by Dirichlet energy and entropy measures without resorting to high-order polynomial expansions. The model fuses filter outputs through a node-adaptive gating mechanism to enable node-wise structure-aware spectral selection. We provide a formal spectral analysis grounded in Dirichlet energy attenuation, von Neumann entropy, and curvature polarity, and derive explicit characterizations of filter behavior across varying levels of homophily and structural perturbations. Extensive benchmarks on 10 datasets show that DCQ-GNN ties for the top average rank (3.0) on heterophilic graphs and obtains the second-best rank (4.2) on homophilic graphs, remaining competitive with representative high-order polynomial spectral filters. Furthermore, under strong structural perturbations, DCQ-GNN exhibits substantially smaller performance degradation compared to both first-order and high-order baselines. These results demonstrate that curvature-aware quadratic banks provide a robust and efficient alternative to high-order spectral models while preserving optimization stability and computational efficiency.
\end{abstract}

\section{Introduction}
Graph neural networks (GNNs) have achieved strong performance across a wide range of relational learning tasks. From a spectral perspective, many GNN architectures can be interpreted as frequency-selective filters applied to the graph Laplacian, where message passing corresponds to polynomial filtering in the spectral domain. This interpretation has inspired a large body of spectral GNN models that explicitly design filter responses over graph frequencies.  Low-order spectral filters are efficient and stable but exhibit limited selectivity: attenuation outside the passband is gradual and may obscure structural signals (Fig.~\ref{fig: response LvsQ}). High-order polynomial filters achieve sharper responses but introduce optimization challenges and sensitivity to structural perturbations.
\begin{figure}[ht]
	\centering
	\includegraphics[width=0.82\linewidth]{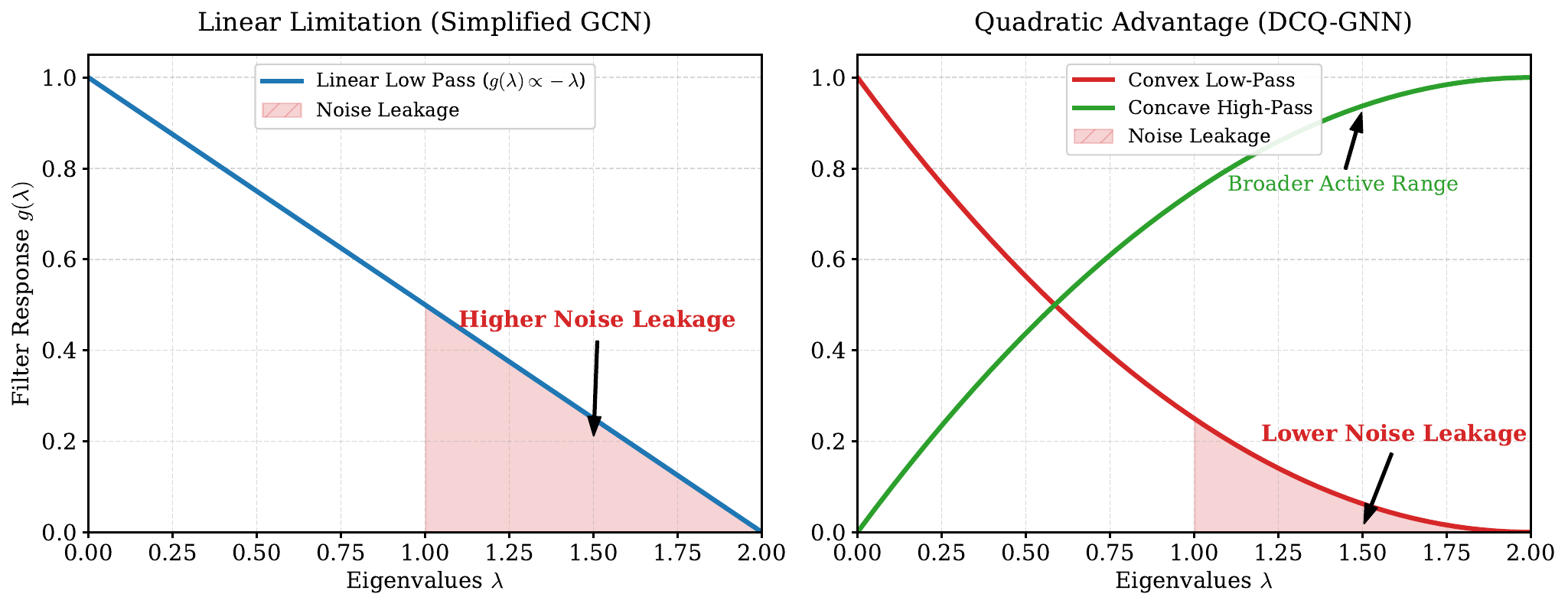}
	\caption{Spectral responses of linear and convex--concave quadratic Laplacian filters over the normalized eigenvalue range $[0,2]$. The quadratic filter exhibits sharper curvature near the cutoff region compared to the linear baseline.}
	\label{fig: response LvsQ}
\end{figure}
Existing spectral GNNs primarily pursue expressivity by increasing polynomial order or expanding the filter basis family. In contrast, we argue that the primary design axis of spectral filters need not be polynomial approximation order, but rather curvature polarity, i.e., explicit control over convexity and concavity of the spectral response. This perspective shifts the goal from universal approximation toward optimizing the stability–selectivity trade-off under low-order constraints. Curvature polarity constrains second derivative sign globally, which cannot be guaranteed in unconstrained polynomial expansions.

We propose DCQ-GNN, a spectral GNN based on a compact bank of adaptive convex--concave quadratic filters. Instead of increasing polynomial degree, DCQ-GNN fixes the filter order to two and introduces complementary curvature polarities to shape spectral attenuation. Pairing convex and concave responses yields sharper spectral transitions than linear filters without the instability of high-order polynomials.  This curvature-centric design distinguishes DCQ-GNN from basis-expansion approaches such as Bernstein- or Jacobi-polynomial filters, which improve spectral flexibility via higher-order polynomial expansions. Our approach instead leverages controlled curvature to enhance spectral selectivity per parameter, prioritizing robustness and stability over universal approximation capacity.

To enable structure-aware spectral adaptation, DCQ-GNN fuses filter outputs via a node-adaptive gating mechanism. This gating module performs localized node-wise spectral selection, allowing different nodes to emphasize convex or concave responses depending on structural context, thereby introducing node-wise spectral bias without increasing filter order. We provide a formal spectral analysis grounded in Dirichlet energy decay gap, von Neumann entropy evolution, and curvature polarity. Our analysis characterizes how convex and concave quadratic components induce distinct attenuation regimes, revealing a controllable mechanism for modulating spectral energy decay under varying homophily levels and structural perturbations. Extensive experiments on ten benchmark datasets demonstrate that DCQ-GNN achieves competitive performance across both homophilic and heterophilic graphs. Notably, under strong structural perturbations, DCQ-GNN exhibits smaller performance degradation compared to both first-order and high-order baselines, supporting our hypothesis that curvature-aware low-order filtering improves robustness.  Our contributions can be summarized as follows:
\begin{itemize}
	\item  We introduce curvature polarity as an explicit constraint for second-order spectral GNN filters. By pairing convex and concave responses, DCQ-GNN spans the full quadratic filter space characterized in Theorem~\ref{thm: app-convex_concave_spanning}, rather than selecting filters by heuristic frequency centers or bandwidths.

	\item We propose a compact convex--concave quadratic filter bank combined with node-adaptive gating.
	\item We provide a theoretical spectral interpretation via Dirichlet energy decay gap and entropy analysis.
	\item We demonstrate competitive performance and improved robustness across diverse graph regimes spanning both homophilic and heterophilic settings,  with computational costs comparable to low-order propagation.
\end{itemize}
\begin{table*}[!h]
	\centering
	\caption{Performance comparison between DCQ-GNN and representative polynomial spectral GNNs, including BernNet and JacobiConv.}
	\begin{tabular}{lccc}
		\toprule
		Criterion           &  BernNet  & JacobiConv & DCQ-GNN\\
		\midrule
		Polynomial Order    & High     & High       & 2    \\
		Curvature Control       & Implicit   &  Implicit  &  Explicit \\
		Stability              & Moderate   &  Moderate   &  High    \\
		Entropy Interpretation  & No        &  No        &  Yes     \\
		Node-Adaptive Selection &  No       &  No         & Yes     \\
		\bottomrule
	\end{tabular}
	\label{tab:compar-high}
\end{table*}
Compared to Bernstein- and Jacobi-based spectral filters that increase polynomial order to enhance flexibility, our approach instead exploits curvature polarity within a fixed second-order formulation (Table~\ref{tab:compar-high}).  Detailed comparisons with Bernstein- and Jacobi-based spectral filters are provided in Appendix~\ref{app:discussion}.
\section{The Convex--Concave Quadratic Filter GNN Architecture}
\label{sec:method}
\begin{figure}[ht]
	\centering
	\includegraphics[width=0.94\linewidth]{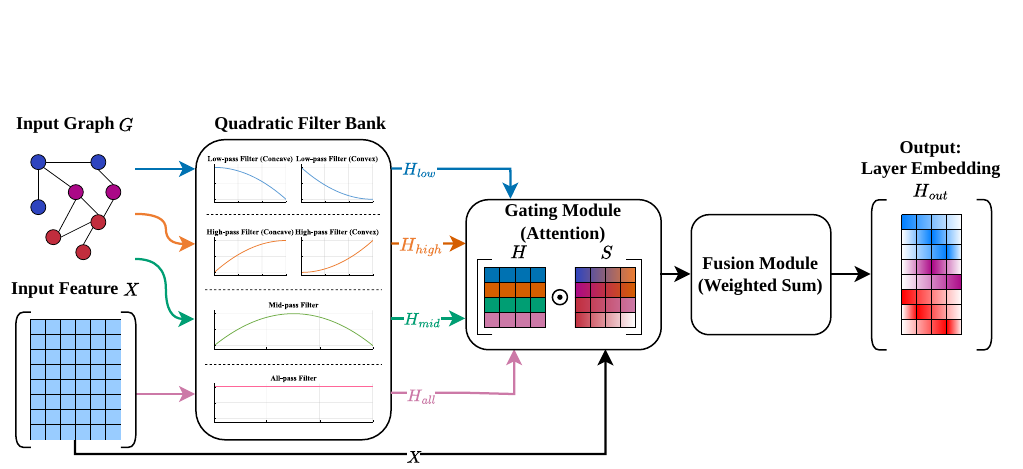}
	\caption{Architecture of DCQ-GNN. Parallel convex and concave quadratic filters extract complementary frequency components, whose outputs are fused through a node-adaptive gating mechanism with an all-pass residual anchor.}
	\label{fig:architecture}
\end{figure}
We propose DCQ-GNN, a spectral graph neural network designed to address the limited frequency selectivity of linear (first-order) Laplacian filters while preserving the computational efficiency and numerical stability of low-order message passing mechanisms.
Rather than increasing polynomial degree, DCQ-GNN enhances spectral expressivity through a compact bank of adaptive quadratic filters, whose convex--concave curvature induces complementary frequency responses within a strictly second-order formulation.

As illustrated in Fig.~\ref{fig:architecture}, each DCQ-GNN layer consists of two components:
\begin{itemize}
	\item \textbf{Adaptive convex--concave quadratic filter bank.}
	A small set of second-order spectral filters that extract low-, mid-, and high-frequency components, together with an all-pass (identity) channel. The filters are parameterized to adapt their transition regions and magnitudes, enabling frequency-aware representations across graphs with diverse homophily characteristics.
	\item \textbf{Node-adaptive gated fusion.}
	A gating mechanism that combines the filtered channels on a per-node basis, allowing the model to emphasize structure-dependent spectral responses when the structure is informative, and to fall back on the all-pass anchor when graph structure is unreliable or noisy.
\end{itemize}
We next formalize the spectral setting, describe the construction of the convex--concave quadratic filter bank, and introduce the node-adaptive fusion mechanism.
\subsection{Preliminaries}
Let $G=(\mathcal{V},\mathcal{E})$ be an undirected graph with $N$ nodes, adjacency matrix $A\in \mathbb{R}^{N \times N}$, and diagonal degree matrix $D$. We utilize the symmetrically normalized Laplacian ${\mathbf L}=I-D^{-1/2}AD^{-1/2}=U{\Lambda} U^\top$, where $0 \le \lambda_1 \le \dots \le \lambda_N \le 2$. For node features $X \in \mathbb{R}^{N \times F}$, a spectral filter $g:[0,2]\rightarrow\mathbb{R}$ is defined as $g(\mathbf L)X= U g(\Lambda)U^\top X$. To ensure computational efficiency and local message passing, we restrict $g$ to quadratic polynomials $g(\lambda)=a_0+a_1\lambda+a_2\lambda^2$ with $a_0$, $a_1$, and $a_2$ constants, which avoids explicit eigendecomposition while enabling nonlinear spectral shaping. DCQ-GNN explicitly leverages the curvature of these second-order filters to refine frequency responses without increasing the polynomial degree.

\subsection{Adaptive Convex--Concave Quadratic Filter Bank}
\label{sec:filter_bank}

We construct a compact filter bank $\mathcal{B}=\{g_{\text{low}}, g_{\text{mid}}, g_{\text{high}}, g_{\text{all}}\}$, where each $g(\lambda) $ is a quadratic polynomial on $[0,2]$. Each low- or high-pass channel contains convex--concave variants, but the gating mechanism aggregates them into four effective frequency channels. To introduce adaptivity with minimal parameterization, we employ two learnable scalar parameters:
(i) a cutoff proxy ${\tilde \lambda} = 2-\alpha$ with $\alpha\in[0,1]$, which controls the spectral location of the quadratic extremum; and
(ii) a positive scale $0<\beta\le \min\{1,{\tilde \lambda}^2/4\}$\footnote{This specific bound on $\beta$ is necessary and sufficient to guarantee  $\max_{\lambda \in [0,2]} |g(\lambda)| \le 1.$}, which controls the response magnitude and ensures well-defined curvature.

\textbf{Filter Channels.}
Let  ${\bf Z}_1= {\bf L} X$, ${\bf Z}_2= {\bf L}{\bf Z}_1={\bf L}^2 X $.  Each channel can be expressed as a linear combination of $X$, ${\bf Z}_1$, and ${\bf Z}_2$, allowing efficient implementation through sparse propagation. 
\begin{enumerate}
	\item  {\bf Low-pass channel (homophily).} Define \textbf{convex low-pass} as
	\begin{equation}\label{cvx-low}
		g_{\text{low}}^{\text{cvx}}( {\bf L}) X=\frac{\beta}{{\tilde \lambda}^2}({\bf Z}_2-2{\tilde \lambda}{\bf Z}_1+{\tilde \lambda}^2 X),
	\end{equation}
	which induces sharper attenuation in the high-frequency regime, and \textbf{concave low-pass} as
	\begin{equation}\label{ccv-low}
		g_{\text{low}}^{\text{ccv}}( {\bf L})X
		=-\frac{\beta}{{\tilde \lambda}^2}({\bf Z}_2-{\tilde \lambda}^2X),
	\end{equation}
	which maintains a broader effective passband across low-to-mid frequencies.
	
	\item   {\bf High-pass channel (heterophily).}  To emphasize heterophilic high-frequency components, we use
	\textbf{convex high-pass}
	\begin{equation}\label{cvx-high}
		g_{\text{high}}^{\text{cvx}}( {\bf L})X=\frac{\beta}{{\tilde \lambda}^2}{\bf Z}_2,
	\end{equation}
	which strongly suppresses low-frequency content, and \textbf{concave high-pass}
	\begin{equation}\label{ccv-high}
		g_{\text{high}}^{\text{ccv}}( {\bf L})X
		=-\frac{\beta}{{\tilde \lambda}^2}({\bf Z}_2-2{\tilde \lambda}{\bf Z}_1)
	\end{equation}
	to yield a wider active region in the high-frequency regime.
	
	\item  {\bf Mid-pass:} $g_{\text{mid}}( {\bf L})X
	=-\frac{4\beta}{{\tilde \lambda}^2}({\bf Z}_2-{\tilde \lambda}{\bf Z}_1)$, a non-monotone response emphasizing intermediate frequencies, which are empirically more robust to structural perturbations \cite{huang2023robust}.
	\item {\bf All-pass (identity):}
	$g_{\text{all}}({\bf L}) X = X$, which preserves raw features and serves as a structure-agnostic anchor.
\end{enumerate}

\paragraph{\textbf {Why quadratic rather than alternative structured bases?}}
	At fixed order $K=2$, Chebyshev and Jacobi bases span the same three-dimensional polynomial space as $a_0+a_1\lambda+a_2\lambda^2$ under an invertible change of basis. They therefore do not enlarge the admissible spectral response class, while making curvature polarity less explicit. Rational filters $g(\lambda)=p(\lambda)/q(\lambda)$ can realize sharper roll-offs, but possible poles introduce additional constraints on $q(\lambda)$ and complicate the uniform boundedness requirement $|g(\lambda)|\le 1$ on $[0,2]$. Wavelet- or framelet-based decompositions provide multi-scale separation, but require frame construction and scale-selection constraints that do not align naturally with the learnable cutoff $\tilde{\lambda}$. The quadratic class is therefore the lowest-degree polynomial family that simultaneously (i) encodes sign-definite curvature through the single coefficient $a_2$, (ii) admits an analytic Lipschitz stability bound under Laplacian perturbations (Theorem~\ref{app:thm-strucprtur}), and (iii) recovers a second-order Taylor form of heat diffusion, which motivates the convex low-pass branch.

The negative signs are introduced to ensure passband alignment. The curvature parameters may also be treated as hyperparameters, allowing the model to adapt to graphs with different homophily characteristics. 
Specifically, the convex low-pass filter defined in Eq. \eqref{cvx-low} yields
$$g_{\text{low}}^{\text{cvx}}(\lambda)
= \frac{\beta}{{\tilde \lambda}^2}\left(\lambda^2 - 2{\tilde \lambda}\lambda + {\tilde \lambda}^2\right),
$$
which  corresponds to a second-order Taylor approximation of the heat kernel $e^{-\tau \lambda}$. Consequently, the convex low-pass filter primarily acts as a denoising operator by attenuating high-frequency components.
On the other hand, When $\tilde{\lambda}=1$ and the maximal admissible scale $\beta=1/4$ is used, the filter pair $\{g_{\text{low}}^{\text{cvx}}, g_{\text{high}}^{\text{ccv}}\}$ constitutes a scaled partition of unity, since
$$g_{\text{low}}^{\text{cvx}}(\lambda) + g_{\text{high}}^{\text{ccv}}(\lambda) = \beta = 1/4, \quad \forall \lambda \in [0, \tilde{\lambda}].$$
This property ensures that the filter bank uniformly covers the spectral domain without frequency gaps, which is a fundamental requirement for stable signal reconstruction in graph signal processing (GSP). Unlike standard GCNs, which employ a fixed linear filter $1-\lambda$, quadratic filters enable frequency-dependent gain modulation across the spectrum. This design is closely related to quadratic mirror filters in classical signal processing, where signals are decomposed into sub-bands using filters with complementary curvatures. From a spectral geometry perspective, the quadratic term introduces controllable curvature in the frequency response, allowing the model to shape the graph signal manifold beyond the linear bias of traditional GCN filters.

\subsection{Node-Adaptive Gated Fusion}
Let $H_{\text{concat}}
=[H_{\text{low}}|H_{\text{mid}}|H_{\text{high}}|H_{\text{all}}]
\in\mathbb{R}^{N\times 4F}$. We compute node-wise gating coefficients $
S
=\mathrm{softmax}(H_{\text{concat}}W_{\text{gate}}+b_{\text{gate}})
$, 
where $W_{\text{gate}}\in\mathbb{R}^{4F\times 4}$ and $b_{\text{gate}}\in\mathbb{R}^{4}$ are learnable parameters. The layer output is
$${H_{\text{out}}}
=\sigma\left(
\sum_{k\in\{\text{low},\text{mid},\text{high},\text{all}\}}
S_k\odot(H_k{\Theta}_k)
\right),
$$
where $H_k\in{\mathbb R}^{N\times F}$ is the output of the $k$-th filter channel, ${\Theta}_k\in\mathbb{R}^{F\times F'}$ is a channel-specific linear matrix, $S_k$ denotes the $k$-th column of $S$  broadcast across feature dimensions, $\odot$ stands for elementwise multiplication, and $\sigma$ denotes an activation function. This mechanism enables a  node-wise spectral mixture model, allowing DCQ-GNN to adaptively down-weight structure-dependent channels and rely on the all-pass anchor when appropriate.
\subsection{Theoretical Analysis}
\label{sec:theoretical_analysis}
DCQ-GNN is based on the observation that quadratic spectral filters can improve frequency selectivity over linear filters while preserving low-order propagation. We analyze this property from two perspectives: high-frequency Dirichlet energy attenuation and spectral concentration. Theorem~\ref{thm:main-diri-energy-rate} shows that, for any suppression band $[\lambda^*,2]$, a bounded strictly quadratic filter can attain a lower worst-case weighted Dirichlet energy than the optimal bounded linear filter $g_\ell(\lambda)=1-a\lambda$. This result explains why the quadratic channels in DCQ-GNN can suppress high-frequency noise without increasing message-passing depth. Theorem~\ref{thm:main-entropyMajor} further shows that a normalized convex quadratic low-pass filter induces lower von Neumann entropy than its linear counterpart, indicating reduced spectral leakage within each low-pass channel before node-adaptive gated fusion.

\subsubsection{High-Frequency Dirichlet Energy Decay Gap.}\label{frequ-diri-gap}
We consider the normalized Laplacian $ \mathbf{L}$ with spectrum $\lambda \in [0,2] $.
For a graph signal $x$, its Dirichlet energy after applying a spectral filter $g$ is:
$$\mathcal{E}_g(x)=\sum_{i=1}^N\lambda_i (g(\lambda_i)\hat{x}_i)^2,$$
with $\hat{x}_i$ the $i$-th component of  $\hat{x} = U^\top x$.  We analyze worst-case attenuation over the high-frequency band
$$
\mathcal{H}_{\lambda^*} := \{\lambda : \lambda \ge \lambda^*\}, \quad \lambda^* \in (0,2).
$$
Then we have the following result.
\begin{theorem}[High-Frequency Dirichlet Energy Decay Gap ]\label{thm:main-diri-energy-rate}
	Let $\mathcal{G}_L = \{1 - a\lambda \mid a \in [0, 1]\}$ denote the class of linear spectral filters, and $\mathcal{G}_Q = \{1 - a\lambda + b\lambda^2 \mid |g(\lambda)| \leq 1, \forall \lambda \in [0, 2]\}$ denote the class of quadratic spectral filters. For any transition frequency $\lambda^* \in (0, 2)$, let $g_\ell^* \in \mathcal{G}_L$ be the optimal linear filter minimizing the maximum weighted Dirichlet energy in the high-frequency band $[\lambda^*, 2]$. Then, there exists a strictly quadratic filter $g_q \in \mathcal{G}_Q$ (where $b \neq 0$) that achieves a strictly lower energy bound:
	$$\sup_{\lambda \in [\lambda^*, 2]} \lambda |g_q(\lambda)|^2 < \sup_{\lambda \in [\lambda^*, 2]} \lambda |g_\ell^*(\lambda)|^2$$
\end{theorem}

Theorem \ref{thm:main-diri-energy-rate} shows that quadratic filters achieve strictly stronger attenuation within the transition band $[\lambda^*, \lambda_{\max}]$ than any linear filter satisfying the same boundedness constraint. This result explains why DCQ-GNN can obtain sharper spectral discrimination without resorting to high-order polynomial expansions. 
Importantly, this improved attenuation arises within a strictly second-order formulation, preserving the stability and efficiency of low-order propagation. The quadratic term introduces controllable curvature that accelerates high-frequency decay while satisfying $|g(\lambda)| \le 1$.

\subsubsection{Spectral Selectivity via von Neumann Entropy (a measure of spectral spread).}
We quantify the spectral footprint of a filter through von Neumann entropy.
Let $P=\frac{g(\mathbf{L})^2}{\mathrm{Tr}(g(\mathbf{L})^2)}$ be the normalized power-response matrix, with eigenvalues $\{\eta_i\}_{i=1}^N$.
The entropy is $S_{\mathrm{VN}}(P)=-{\mathrm Tr}(P\ln P)$, lower values indicate a more concentrated (more selective) spectral response. Since $\mathbf L$ is diagonalizable, the entropy reduces to a spectral distribution over eigenvalues, allowing the analysis to be expressed directly in terms of the spectral response $g(\lambda)$. This leads to the following result.

\begin{theorem}[Entropy Majorization]\label{thm:main-entropyMajor}
	Let $g_1(\lambda) = \alpha_1(\lambda_{\max}-\lambda)$ and $g_2(\lambda) = \alpha_2(\lambda_{\max}-\lambda)^2$ be a linear and a convex quadratic low-pass spectral filter with $\lambda_{\max}=2$, respectively, where $\alpha_1, \alpha_2$ are normalization constants such that $\sum_i g^2(\lambda_i) = 1$. Let ${\bf p}, {\bf q}\in \mathbb{R}^N$ be the resulting normalized power spectral distributions, where $p_i = g_1^2(\lambda_i)$ and $q_i = g_2^2(\lambda_i)$. Then, the quadratic filter induces a more concentrated spectral response in the sense of von Neumann entropy:
	$$S_{\mathrm{VN}}(\bf q) \leq S_{\mathrm{VN}}(\bf p)$$ with equality if and only if all non-zero eigenvalues of the Laplacian are identical.
\end{theorem}
DCQ-GNN uses a filter bank rather than a single operator. Lower entropy for an individual channel indicates reduced spectral leakage within that channel, facilitating cleaner decomposition before gated recombination.  Intuitively, the quadratic spectral decay concentrates energy more strongly near the passband. Hence, quadratic filters induce lower spectral entropy, indicating stronger concentration of spectral energy. However, the adaptive gating mechanism is what prevents this strict concentration from becoming a liability on heterophilic datasets. The proof of Theorem~\ref{thm:main-entropyMajor} is provided in Appendix~\ref{app:theory-proof}.

\subsection{Interpretive Analysis}
In this section, we analyze the robustness properties of DCQ-GNN against structural perturbations, followed by an explanation of its performance advantages. We first demonstrate how the quadratic filter bank intrinsically dampens the impact of adversarial edge injections.

\begin{proposition}[Structural Interference Cancellation] \label{prop:main-sturcInfer} 
	Let $\mathbf L$ be the symmetrically normalized Laplacian and $\hat A= I - \mathbf L$ be the normalized adjacency matrix.  Consider a target node $u$ corrupted by an adversarial edge $(u, v)$ in the regime of homophilic topological noise, where the neighbors' features of the attacker $v$  satisfy   $x_t = x_v + \xi_t$ with $\|\xi_t\| \le \epsilon$ for all $t \in \mathcal{N}(v)$. Assume the degrees of $v$ and its neighbors are approximately equal ($d_t \approx d_v$). Under the simplified concave quadratic filter $g(\mathbf L) = 2\mathbf L - {\mathbf L}^2 = I - \hat{A}^2$, the direct first-order adversarial pollution $\Delta_{GCN} = \frac{1}{\sqrt{d_u d_v}}x_v$ is perfectly neutralized by the second-order random walk component, yielding a residual representation error strictly bounded by:
	$$\|\text{Err}\| \le \epsilon \sqrt{\frac{d_u}{d_v}}.$$
	Consequently, as the local homophily around the attacker increases ($\epsilon \rightarrow 0$), the quadratic filter exactly cancels the structural interference.
\end{proposition}
Proposition~\ref{prop:main-sturcInfer} reveals that the concave quadratic filter effectively subtracts two-hop aggregated features from the original node features. The second-order term $\hat{A}^2$ provides a restorative counter-signal that neutralizes the first-order malicious injection. The assumption  in Proposition~\ref{prop:main-sturcInfer} represents a worst-case adversarial scenario where the attacker node lies inside a locally homophilic cluster. The analysis therefore isolates the structural effect of the quadratic filter rather than modeling arbitrary feature noise.
\begin{figure}[!htp]
	\centering	\includegraphics[width=0.89\linewidth]{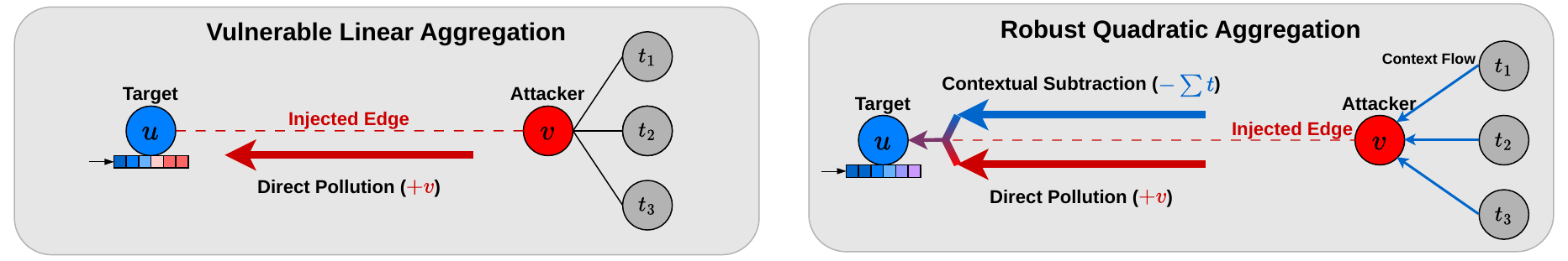}
	\caption{Illustration of structural interference cancellation. While a first-order GCN propagates adversarial noise from attacker node $v$, the quadratic term in DCQ-GNN introduces a counteracting signal that helps preserve the feature integrity of the target node $u$.}
\end{figure}

\subsubsection{Performance Explanation}
The performance advantages of DCQ-GNN arise from combining filters with complementary curvature polarities, which is described in the 
following definition. Combining curvature polarities allows the model to approximate band-pass behaviors unattainable by identical polarities. Detailed discussion is presented in Appendix  \ref{app: perfor-exp}.
\begin{definition}[Curvature Polarity]\label{def:main-curPol}
	For a second-order spectral filter $g(\lambda)$, we define its curvature polarity $\mathcal{CP}$ as the sign of its second-order derivative with respect to the eigenvalue $\lambda$:
	$$\mathcal{CP}(g) := \text{\rm sgn}\left( \frac{d^2 g(\lambda)}{d \lambda^2} \right).$$
	We say $g(\lambda)$ has positive polarity ($\mathcal{CP} > 0$) if it is strictly convex, and negative polarity ($\mathcal{CP} < 0$) if it is strictly concave.
\end{definition}
As shown in Appendix \ref{app: perfor-exp}, combining filters with opposing polarities allows DCQ-GNN to approximate non-monotonic band-pass behaviors that are unattainable by identical-polarity combinations.

\section{Experiments}
We evaluate the proposed DCQ-GNN on semi-supervised node classification tasks\footnote{Code is available in \href{github.com/yanrh1999/DCQ-GNN}{github.com/yanrh1999/DCQ-GNN}} and investigate the following research questions: (1) \textbf{RQ1 (Expressivity):} Does the proposed convex--concave quadratic filter bank provide stronger generalization capability across the homophily--heterophily spectrum compared to linear and higher-order polynomial baselines? (2) \textbf{RQ2 (Robustness):} Does the proposed structural interference cancellation mechanism enhance model stability under adversarial structural perturbations? (3) \textbf{RQ3 (Ablation Study):} How do different frequency channels (low-pass, mid-pass, high-pass, and all-pass) and the gating mechanism contribute to performance across varying graph structural regimes? (4) \textbf{RQ4 (Complexity and Performance Analysis):} Why does the quadratic filter bank provide a better efficiency–performance trade-off than learned high-order polynomial filters?

\subsection{Experimental Setup} 

\paragraph{Datasets.}
We evaluate DCQ-GNN on $10$ datasets that span a broad homophily–heterophily spectrum, including citation (CiteSeer, PubMed), co-purchase (Photo, Computers), and heterophilic graphs (Texas, Actor, Squirrel, etc.). Detailed dataset statistics and homophily measures are provided in Appendix \ref{app:experiments:data}.

\paragraph{Baselines.} 
Our method is compared against four categories of state-of-the-art GNNs: (1) Classical convolution and attention models (GCN~\cite{kipf2017semi}, GAT~\cite{velivckovic2018graph}); (2) Deep/residual architectures (JK-Net~\cite{xu2018representation},  GCNII~\cite{chen2020simple}); and (3) Representative polynomial/spectral filters (H2GCN~\cite{zhu2020beyond}, FAGCN~\cite{bo2021beyond}, BernNet~\cite{he2021bernnet}, EvenNet~\cite{lei2022evennet}, JacobiConv~\cite{wang2022powerful}), and (4) Kernel-based LHKGNN \cite{qin2025localized}. 

\paragraph{Experiment settings.}
We follow standard semi-supervised node classification protocols with $60\%/20\%/20\%$ train/validation/test splits\footnote{Partitions are generated via simple random sampling (SRS). We do not enforce class stratification or a minimum sample count per class in the training set.}. Hyperparameters for all models are optimized using Optuna \cite{optuna_2019} under equivalent search budgets to ensure a fair comparison.

\subsection{Performance on Node Classification (RQ1)} 
\begin{table*}[ht] 
	\centering
	\caption{Node classification accuracy (mean $\pm$ standard deviation) on $10$ benchmark datasets. Results are averaged over 10 independent runs. The best result is highlighted in \textbf{bold} and the second best is \underline{underlined}.}
	\resizebox{0.98\textwidth}{!}{ 
		\begin{tabular}{lcccccc} 
			\toprule
			\textbf{Datasets}
			& \textbf{Photo} & \textbf{PubMed} & \textbf{Computers} & \textbf{CiteSeer} & \textbf{WikiCS} & \textbf{Avg. Rank} \\ 
			\textbf{Homophily (Node)} & $h=0.83$ & $h=0.79$ & $h=0.78$ & $h=0.71$ & $h=0.65$ & {} \\ 
			\midrule 
			GCN & $94.10 \pm 0.40$ & $88.53 \pm 0.28$ & $90.64 \pm 0.38$ & $76.03 \pm 1.13$ & $82.02 \pm 0.69$ & 9.0 \\
			GCNII & $\mathbf{95.80 \pm 0.50}$ & $\mathbf{90.32 \pm 0.30}$ & $\mathbf{92.19 \pm 0.39}$ & $76.23 \pm 1.83$ & $\underline{84.24 \pm 0.72}$ & $\mathbf{1.8}$ \\
			GAT & $94.61 \pm 0.48$ & $88.12 \pm 0.48$ & $90.16 \pm 0.89$ & $75.82 \pm 0.95$ & $84.23 \pm 0.82$ & 7.4 \\
			JK-Net & $94.33 \pm 0.57$ & $88.49 \pm 0.46$ & $91.50 \pm 0.55$ & $76.47 \pm 1.25$ & $83.43 \pm 1.14$ & 6.4 \\
			H2GCN & $94.92 \pm 0.53$ & $89.52 \pm 0.47$ & $90.78 \pm 0.57$ & $73.33 \pm 0.96$ & $83.40 \pm 0.76$ & 7.2 \\
			EvenNet & $94.32 \pm 0.43$ & $89.56 \pm 0.43$ & $\underline{91.73 \pm 0.57}$ & $73.96 \pm 1.12$ & $83.47 \pm 0.73$ & 6.0 \\
			FAGCN & $95.24 \pm 0.31$ & $89.21 \pm 0.55$ & $91.25 \pm 0.78$ & $73.30 \pm 0.61$ & $83.58 \pm 1.03$ & 6.2 \\
			JacobiConv & $94.83 \pm 0.23$ & $89.44 \pm 0.10$ & $91.20 \pm 0.18$ & $74.75 \pm 1.21$ & $83.10 \pm 0.47$ & 7.2 \\
			BernNet & $94.17 \pm 0.48$ & $89.46 \pm 0.33$ & $89.56 \pm 0.78$ & $\mathbf{76.73 \pm 0.97}$ & $83.80 \pm 0.75$ & 6.2 \\
			LHK-GNN & $\underline{95.65 \pm 0.37}$ & $88.71 \pm 0.38$ & $91.27 \pm 0.72$ & $\underline{76.55 \pm 1.16}$ & $83.56 \pm 0.64$ & 4.4 \\ \hline
			\textbf{DCQ-GNN (ours)} & $95.07 \pm 0.51$ & $\underline{89.62 \pm 0.49}$ & $90.90 \pm 0.50$ & $74.89 \pm 1.25$ & $\mathbf{85.50 \pm 0.82}$ & $\underline{4.2}$ \\ \hline
			\toprule 
			\textbf{Datasets} & \textbf{Wisconsin} & \textbf{Actor} & \textbf{Chameleon} & \textbf{Squirrel} & \textbf{Texas} & \textbf{Avg. Rank} \\ 
			\textbf{Homophily (Node)} & $h=0.17$ & $h=0.15$ & $h=0.10$ &$h=0.08$ & $h=0.06$ & {} \\ 
			\midrule 
			GCN & $47.20 \pm 5.60$ & $27.45 \pm 1.02$ & $41.34 \pm 2.08$ & $29.13 \pm 1.65$ & $56.76 \pm 7.74$ & 10.6 \\
			GCNII & $78.60 \pm 6.87$ & $\mathbf{37.43 \pm 1.38}$ & $52.90 \pm 1.30$ & $36.73 \pm 1.18$ & $80.00 \pm 7.07$ & 6.0 \\
			GAT & $47.80 \pm 6.16$ & $29.14 \pm 1.18$ & $44.29 \pm 2.43$ & $30.62 \pm 0.84$ & $55.68 \pm 6.42$ & 9.8 \\
			JK-Net & $49.40 \pm 4.57$ & $30.53 \pm 1.03$ & $42.29 \pm 2.20$ & $27.74 \pm 1.87$ & $57.30 \pm 8.95$ & 9.6 \\
			H2GCN & $80.80 \pm 2.54$ & $34.76 \pm 1.06$ & $58.91 \pm 1.06$ & $37.40 \pm 1.39$ & $74.47 \pm 6.04$ & 6.8 \\
			EvenNet & $81.80 \pm 2.82$ & $34.36 \pm 0.65$ & $\underline{65.02 \pm 1.77}$ & $49.71 \pm 0.85$ & $\mathbf{83.46 \pm 2.04}$ & 3.8 \\
			FAGCN & $\underline{83.20 \pm 4.40}$ & $32.62 \pm 1.44$ & $60.47 \pm 3.09$ & $47.42 \pm 1.83$ & $80.25 \pm 3.32$ & 5.4 \\
			JacobiConv & $75.60 \pm 5.45$ & $35.80 \pm 0.72$ & $\mathbf{66.34 \pm 1.06}$ & $\mathbf{53.17 \pm 0.93}$ & $78.91 \pm 3.78$ & 4.0 \\
			BernNet & $\mathbf{85.20 \pm 1.40}$ & $35.44 \pm 0.92$ & $61.19 \pm 0.94$ & $50.82 \pm 0.69$ & $\underline{83.24 \pm 4.32}$ & $\underline{3.4}$ \\ 
			LHK-GNN & $81.40 \pm 4.64$ & $\underline{36.45 \pm 1.29}$ & $64.58 \pm 1.37$ & $51.41 \pm 0.66$ & $82.62 \pm 4.81$ & 3.6 \\ \hline
			\textbf{DCQ-GNN (ours)} & $82.20 \pm 4.77$ & $35.63 \pm 1.39$ & $64.88 \pm 2.31$ & $\underline{52.26 \pm 1.49}$ & $82.97 \pm 5.55$ & $\mathbf{3.0}$ \\
			\bottomrule 
		\end{tabular}
	}
	\label{tab:main_results} 
\end{table*}
Table~\ref{tab:main_results} reports node-classification accuracy on $10$ datasets covering homophilic and heterophilic regimes. On homophilic datasets, DCQ-GNN obtains the second-best average rank ($4.2$), outperforming polynomial spectral baselines such as EvenNet ($6.0$) and BernNet ($6.2$). Although GCNII ranks first ($1.8$), DCQ-GNN remains competitive, ranking first on WikiCS and second on PubMed. These results indicate that curvature-aware quadratic filtering preserves homophilic smoothing while enabling node-wise spectral adaptation.

On heterophilic datasets, DCQ-GNN matches the best average rank ($3.0$) and outperforms high-order spectral baselines including BernNet ($3.4$), EvenNet ($3.8$), and JacobiConv ($4.0$). On Squirrel, DCQ-GNN attains the second-best accuracy ($52.26\%$), slightly below JacobiConv ($53.17\%$). This gap suggests that Squirrel may require finer spectral resolution than globally parameterized quadratic channels can provide. Nevertheless, DCQ-GNN achieves competitive heterophilic performance with a strictly second-order formulation, whereas JacobiConv and BernNet rely on higher-degree polynomial expansions to model complex high-frequency components. These results support the efficiency--resolution trade-off underlying DCQ-GNN: curvature-controlled quadratic filters offer robust performance across homophily regimes, while highly irregular datasets may benefit from piecewise or higher-resolution spectral parameterizations.

\subsection{Robustness Analysis (RQ2)} 
\begin{table}[h] 
	\centering 
	\caption{Accuracy degradation under the DICE structural attack on \textbf{WikiCS}. $\downarrow \Delta\%$ indicates the relative performance drop compared to the clean graph.}
	\resizebox{\textwidth}{!}{
		\begin{tabular}{lcccc} 
			\toprule 
			& \multicolumn{4}{c}{\textbf{DICE Perturbation Ratio}} \\ \cmidrule(lr){2-5} \textbf{Model} & \textbf{0\% (Clean)} & \textbf{1\%} & \textbf{5\%} & \textbf{10\%} \\ 
			\midrule 
			GCN & $82.02 \pm 0.69$ & $75.74 \pm 1.57 (\downarrow7.66\%)$ & $67.58 \pm 1.66 (\downarrow17.60\%)$ & $60.59 \pm 1.99 (\downarrow26.13\%)$ \\ 
			GAT & $84.23 \pm 0.82$ & $81.75 \pm 0.87 (\downarrow2.95\%)$ & $74.67 \pm 0.93 (\downarrow11.35\%)$ & $68.07 \pm 1.61 (\downarrow19.18\%)$ \\ 
			FAGCN & $83.58 \pm 1.03$ & $79.21 \pm 1.06 (\downarrow5.23\%)$ & $72.16 \pm 1.24 (\downarrow13.66\%)$ & $62.87 \pm 1.18 (\downarrow24.78\%)$ \\ 
			BernNet & $83.80 \pm 0.75$ & $77.01 \pm 1.00 (\downarrow8.10\%)$ & $74.44 \pm 1.73 (\downarrow11.17\%)$ & $71.87 \pm 2.71 (\downarrow14.24\%)$ \\ 
			JacobiConv & $83.10 \pm 0.47$ & $80.17 \pm 0.91 (\downarrow3.52\%)$ & $76.86 \pm 1.26 (\downarrow7.51\%)$ & $69.74 \pm 1.06 (\downarrow16.08\%)$ \\  		
			\hline
			\textbf{DCQ-GNN (ours)} & $85.50 \pm 0.82$ & $\mathbf{84.30 \pm 0.83 (\downarrow1.40\%)}$ & $\mathbf{80.40 \pm 0.78 (\downarrow5.96\%)}$ & $\mathbf{77.58 \pm 0.68 (\downarrow9.26\%)}$ \\
			\bottomrule 
	\end{tabular}}
	\label{tab:dice_performance} 
\end{table}
To evaluate robustness against adversarial topological noise, we conduct experiments on the WikiCS dataset (moderate node homophily, $h=0.65$) to the DICE \cite{zugner2018adversarial} structural attack. Table~\ref{tab:dice_performance} reports the node classification accuracy and the corresponding relative performance degradation under increasing edge perturbation ratios ($1\%$, $5\%$, and $10\%$). 

Across all attack  levels, DCQ-GNN consistently exhibits the highest resilience. Under the most extreme scenario ($10\%$ perturbation), DCQ-GNN incurs a minimal relative decrease of only $9.26\%$. In  contrast, standard spatial methods like GCN and FAGCN suffer severe degradation, dropping by $26.13\%$ and $24.78\%$, respectively. Furthermore, DCQ-GNN proves significantly more stable than advanced spectral baselines under the same conditions, outperforming JacobiConv, which experiences a $16.08\%$ performance drop.  These results validate the structural interference cancellation mechanism. The concave quadratic component incorporates two-hop contexts, dampening spurious adversarial edges lacking local support.

\subsection{Ablation Study (RQ3)} 
\begin{table}[h] 
	\centering 
	\caption{Ablation study showing the impact of removing individual channels on the test set. ``w/o'' denotes removal of the corresponding component. Values report accuracy $\pm$ standard deviation, with relative changes in parentheses.}
	\resizebox{\textwidth}{!}{
		\begin{tabular}{lcccc} 
			\toprule 
			\textbf{Datasets} & \textbf{Computers} & \textbf{WikiCS} & \textbf{Wisconsin} & \textbf{Squirrel} \\ 
			\midrule 
			\textbf{Full Model} & $90.90 \pm 0.50$ & $85.50 \pm 0.82$ & $82.20 \pm 4.77$ & $52.26 \pm 1.49$ \\ 
			\midrule 
			w/o low-pass & $82.82 \pm 3.40 (\downarrow 8.89\%)$ & $82.67 \pm 1.01 (\downarrow 3.31\%)$ & $77.80 \pm 3.76 (\downarrow 5.35\%)$ & $50.49 \pm 1.03 (\downarrow 3.38\%)$ \\ 
			w/o mid-pass & $91.22 \pm 0.44 (\uparrow 0.35\%)$ & $84.41 \pm 1.33 (\downarrow 1.28\%)$ & $76.40 \pm 5.56 (\downarrow 7.06\%)$ & $49.58 \pm 2.02 (\downarrow 5.12\%)$ \\
			w/o high-pass & $91.05 \pm 0.36 (\uparrow 0.17\%)$ & $84.48 \pm 1.57 (\downarrow 1.19\%)$ & $71.20 \pm 5.87 (\downarrow 13.38\%)$ & $46.01 \pm 2.53 (\downarrow 11.96\%)$ \\ 
			w/o all-pass & $90.85 \pm 0.56 (\downarrow 0.05\%)$ & $84.17 \pm 1.19 (\downarrow 1.56\%)$ & $73.60 \pm 4.80 (\downarrow 10.46\%)$ & $51.86 \pm 2.21 (\downarrow 0.77\%)$ \\ 
			w/o gating module & $83.82\pm0.64(\downarrow7.90\%)$ & $84.18\pm0.72(\downarrow1.54\%)$ & $79.00\pm5.23(\downarrow3.89\%)$ & $46.51\pm1.98(\downarrow11.00\%)$ \\
			fixed $\alpha, \beta$ & $90.12\pm0.59(\downarrow0.86\%)$ &  $85.91\pm0.66(\uparrow0.48\%)$  & $80.60\pm7.61(\downarrow1.95\%)$ & $48.77\pm1.77(\downarrow6.68\%)$ \\
			\bottomrule
	\end{tabular}}
	\label{tab:ablation} 
\end{table} 

\paragraph{Spectral Adaptivity.}
Table~\ref{tab:ablation} quantifies the contribution of each spectral channel. On the homophilic Computers dataset, removing the low-pass channel causes the largest degradation among channel ablations ($\downarrow 8.89\%$), indicating that local smoothing remains the dominant signal in this regime. In contrast, removing the mid-pass or high-pass channel yields small gains on Computers ($\uparrow 0.35\%$ and $\uparrow 0.17\%$), suggesting that higher-frequency components may introduce redundant or noisy information on strongly homophilic graphs. On heterophilic datasets, the high-pass channel becomes essential: its removal causes the largest channel-level drops on Wisconsin ($\downarrow 13.38\%$) and Squirrel ($\downarrow 11.96\%$). These results support the intended role of the filter bank: low-pass components capture homophilic smoothing, whereas high-pass components preserve discriminative signals when adjacent nodes often belong to different classes.

\paragraph{Role of the All-pass Channel.}
The all-pass channel preserves raw node features and provides a structure-agnostic reference. Its removal has almost no effect on Computers ($\downarrow 0.05\%$), where neighborhood aggregation is reliable, and only a small effect on Squirrel ($\downarrow 0.77\%$), where other spectral channels can partially compensate. The largest drop occurs on Wisconsin ($\downarrow 10.46\%$), indicating that direct feature preservation is critical when topology-driven aggregation is unstable or weakly aligned with labels. Thus, the all-pass channel acts as a fallback pathway, but its contribution depends on the relative informativeness of node attributes and graph structure.

\begin{figure}
	\centering
	\includegraphics[width=0.8\linewidth]{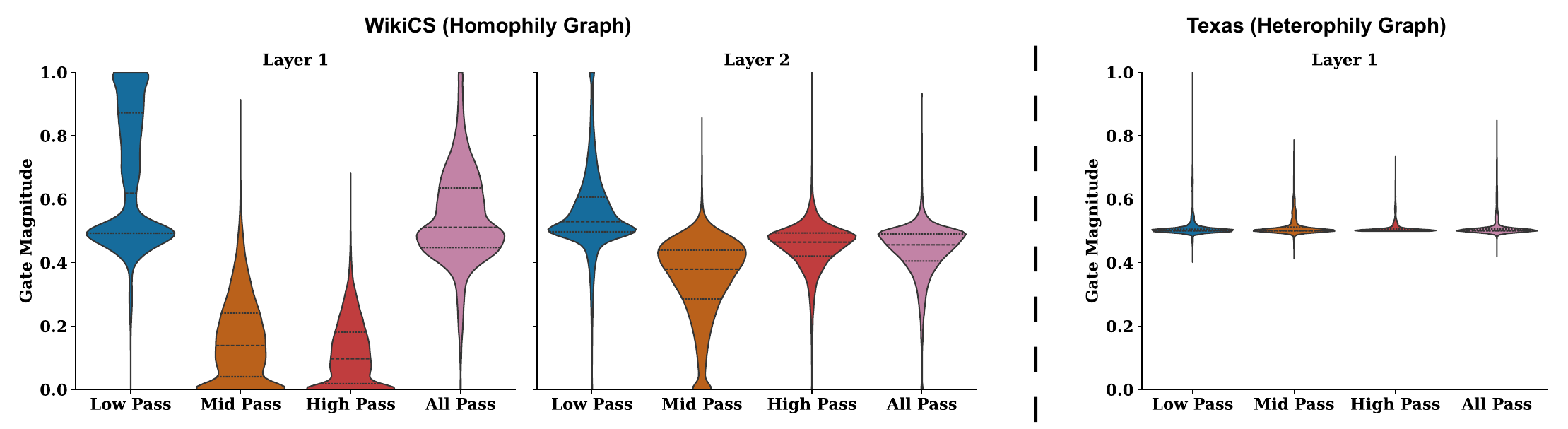}
	\caption{Distribution of learned filter gate magnitudes on the WikiCS and Texas datasets.}
	\label{fig:gate_distribution}
\end{figure}
\begin{figure}
	\centering
	\includegraphics[width=0.8\linewidth]{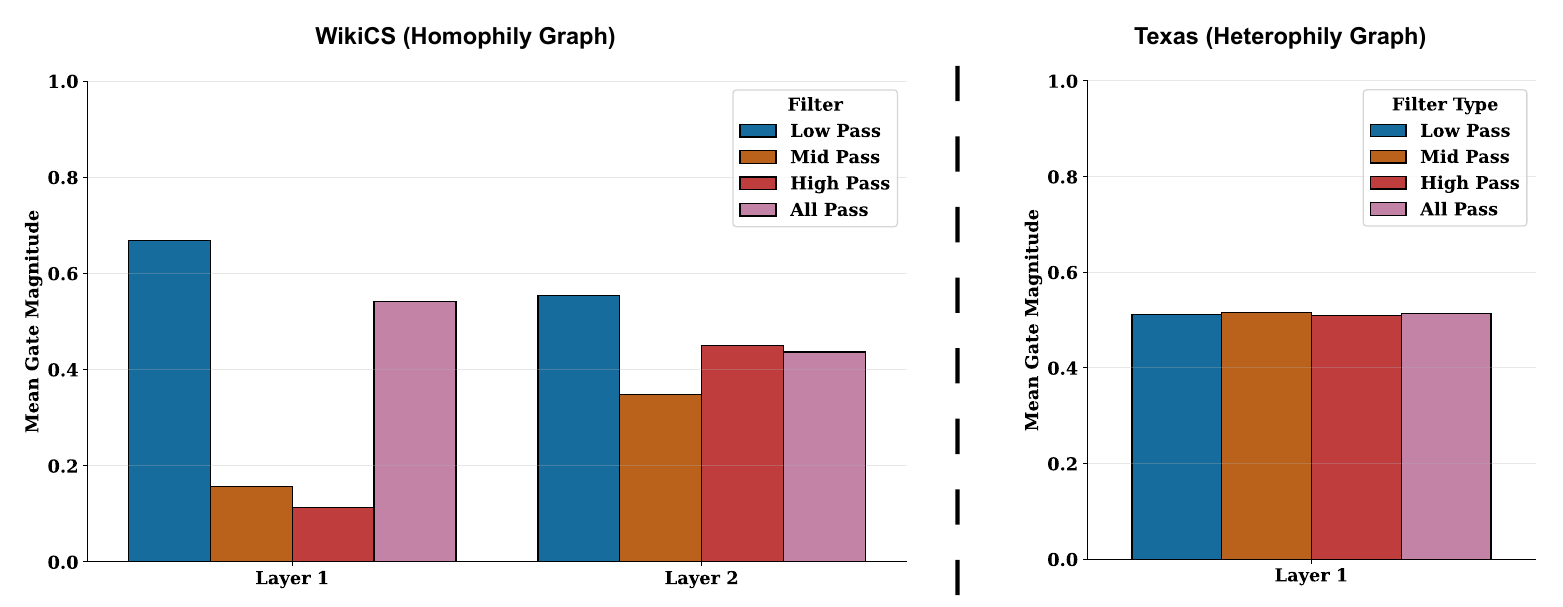}
	\caption{Mean filter gate magnitude across frequency channels on the WikiCS and Texas datasets.}
	\label{fig:gate_activation}
\end{figure}
\paragraph{Impact of the Gating Module.}
Removing the gating module degrades performance on all four datasets, with the largest drops on Squirrel ($\downarrow 11.00\%$) and Computers ($\downarrow 7.90\%$). This confirms that a fixed global mixture of spectral channels is insufficient for graphs with heterogeneous local structures. On Squirrel, the degradation caused by removing the gate is close to that caused by removing the high-pass channel ($\downarrow 11.96\%$), indicating that node-wise gating is needed to route high-frequency information effectively. Figures~\ref{fig:gate_distribution} and~\ref{fig:gate_activation} further illustrate this behavior: on the more homophilic WikiCS dataset, the gate assigns larger weights to low- and all-pass channels, whereas on the heterophilic Texas dataset, the gate distribution becomes more balanced across frequency channels. This pattern is consistent with the model's need to broaden its spectral mixture under stronger heterophily.

\paragraph{Functional Impact of Learnable Filter Parameters.}
The learnable parameters $\alpha$ and $\beta$ control the cutoff proxy and response scale of the quadratic filters. Fixing them has limited effect on Computers ($\downarrow 0.86\%$), Wisconsin ($\downarrow 1.95\%$), and even improves WikiCS slightly ($\uparrow 0.48\%$), suggesting that these datasets do not always require fine adjustment of the filter shape. In contrast, the same fixed-parameter setting causes a clear drop on Squirrel ($\downarrow 6.68\%$), where the spectral structure is more irregular. Thus, learning $\alpha$ and $\beta$ is most beneficial when the graph requires dataset-specific spectral boundaries, while its effect can be modest on graphs whose useful frequency bands are easier to separate.

\subsection{Complexity and Performance Analysis (RQ4)}
\begin{table}[h]
	\centering
	\caption{Computational efficiency versus performance on the heterophilic \textbf{Actor} dataset. DCQ-GNN achieves a favorable trade-off between training time (ms per epoch) and accuracy across varying model depths $T$.}
	\begin{tabular}{lccc}
		\toprule
		\textbf{Models} & \textbf{$\#$Parameters} & \textbf{Time} (ms) & \textbf{Accuracy} (\%) \\
		\midrule
		JacobiConv & $7,087,931$ & $12.30$ & $\mathbf{35.80}$ \\
		BernNet & $60,058$ & $8.15$ & $35.44$ \\
		\midrule
		DCQ-GNN ($T=5$) & $165,352$ & $9.90$ & $34.68$ \\
		DCQ-GNN ($T=4$) & $144,673$ & $7.75$ & $34.23$ \\
		DCQ-GNN ($T=3$) & $123,994$ & $5.88$ & $35.47$ \\
		DCQ-GNN ($T=2$) & $103,315$ & $4.83$ & $35.07$ \\
		DCQ-GNN ($T=1$) & $82,636$ & $\mathbf{3.27}$ & $\mathbf{35.63}$ \\
		\bottomrule
	\end{tabular}
	\label{tab:complex_compa}
\end{table}
To validate the design choice of employing a decoupled quadratic filter bank over high-order single-filter methods, we compare the computational efficiency\footnote{Runtime evaluations are conducted on a server equipped with an Intel Xeon Platinum 8352V CPU (12 vCPUs @ 2.10GHz) and an NVIDIA L20 GPU.} and performance of DCQ-GNN against state-of-the-art baselines: BernNet and JacobiConv, on the \textit{Actor} dataset. Table~\ref{tab:complex_compa} summarizes the model size (parameters), runtime, and classification accuracy.

\textbf{Efficiency vs. Accuracy:} As shown in Table~\ref{tab:complex_compa}, while \texttt{JacobiConv} achieves high accuracy, it incurs a significant computational overhead, requiring over $7$ million parameters and $12.30$ ms per epoch due to its complex polynomial basis. In contrast, DCQ-GNN ($T=3$) reduces the parameter count by an order of magnitude ($>10\times$) and decreases latency by more than $50\%$ ($5.88$ ms vs. $12.30$ ms), all while maintaining highly competitive accuracy ($35.47\%$ vs. $35.80\%$).

\textbf{Scalability:} Notably, DCQ-GNN remains faster than \texttt{JacobiConv} even when the filter bank depth is increased to $T=5$ ($9.90$ ms). Furthermore, compared to \texttt{BernNet}, DCQ-GNN ($T=3$) is approximately $38\%$ faster ($5.88$ ms vs. $8.15$ ms). This demonstrates that our decoupled quadratic formulation effectively aggregates long-range signals (theoretically spanning $2T$-hops) without the heavy computational burden associated with high-degree polynomial expansions.

\textbf{Robustness to Over-smoothing:} Furthermore, the results demonstrate that DCQ-GNN effectively mitigates the over-smoothing problem. In many spectral GNNs, increasing the propagation depth typically leads to a sharp decline in performance as node representations become indistinguishable. However, as shown in Table \ref{tab:complex_compa}, DCQ-GNN maintains stable accuracy as the filter bank depth increases from $T=1$ to $T=5$. Since a depth of $T$ theoretically corresponds to a receptive field of $2T$, our model maintains stable accuracy up to  $10$-hop receptive fields, suggesting improved resistance to over-smoothing.
A discussion of receptive-field expansion and over-squashing is provided in Appendix~\ref{app:oversquashing}.

\section{Conclusion and Future Work}
This work establishes quadratic spectral filtering as an intermediate design between the limited selectivity of linear filters and the optimization complexity of high-order polynomials. We propose DCQ-GNN, which leverages a compact bank of adaptive convex--concave quadratic filters to achieve sharper spectral shaping and improved noise rejection while maintaining a strictly two-hop computational footprint. Our theoretical analysis, grounded in Dirichlet energy and von Neumann entropy, shows that exploiting curvature polarity yields more concentrated spectral responses than first-order baselines. Furthermore, we show that the quadratic formulation inherently facilitates structural interference cancellation, providing a second-order correction mechanism against adversarial edge injections.\\

Empirical results show that DCQ-GNN remains competitive with representative high-order polynomial spectral filters while improving robustness under structural perturbations and reducing model complexity. The fixed second-order parameterization may be less suitable for tasks requiring highly localized, multi-modal, or non-polynomial spectral responses. For irregular spectra such as Squirrel, piecewise quadratic filtering could partition $[0,2]$ into sub-intervals and learn separate second-order responses, retaining curvature control while improving spectral localization. Future work includes robustness bounds under non-stochastic structural noise and scalable curvature-aware filter banks.

\bibliographystyle{plain}
\bibliography{references}
\newpage

\appendix
\section*{Appendices}
\section{Notation}\label{app:notation} 
Table~\ref{app: tabNotation} summarizes the key notations and symbols used throughout this paper.
\begin{table*}[!h]
	\renewcommand{\arraystretch}{1.2}
	\centering
	\caption{Key notations used throughout the paper.}
	\begin{tabular}{l p{12cm}}
		\toprule
		\textbf{Symbol} & \textbf{Description} \\
		\midrule
		$G = (\mathcal{V}, \mathcal{E})$ & An undirected graph with node set $\mathcal{V}$ and edge set $\mathcal{E}$. \\
		$N$ & Number of nodes in the graph, i.e., $|\mathcal{V}|$. \\
		$M$  & Number of edges, i.e., $|\mathcal{E}|$.\\
		$F$  & Dimension of node feature vectors. \\
		$A \in \mathbb{R}^{N \times N}$ & The adjacency matrix of the graph. \\
		$D \in \mathbb{R}^{N \times N}$ & The diagonal degree matrix, where $D_{ii} = \sum_j A_{ij}$. \\
		$\hat{A}$ & Symmetrically normalized adjacency matrix, $\hat{A} = D^{-1/2}AD^{-1/2}$. \\
		$I$ & Identity matrix of size $N \times N$. \\
		$\mathbf{L}$ & The normalized graph Laplacian, $\mathbf{L} = I- \hat{A}$. \\
		$T$   & Depth of the neural network.\\
		$U, \Lambda$ & Eigenvector matrix and diagonal eigenvalue matrix of $\mathbf{L}$, respectively.\\
		$\lambda_i$ & The $i$-th eigenvalue of the Laplacian $\mathbf{L}$, ordered $0 \le \lambda_1 \le \dots \le \lambda_N \le 2$. \\
		$X \in \mathbb{R}^{N \times F}$ & The input node feature matrix. \\
		$g(\lambda)$ & A spectral filter function defined on the spectrum of $\mathbf{L}$. \\
		$\mathcal{B}$ & The proposed convex--concave quadratic filter bank $\{g_{\text{low}}, g_{\text{mid}}, g_{\text{high}}, g_{\text{all}}\}$. \\
		$\alpha$ & Learnable parameter controlling the transition frequency (cutoff) of the quadratic filter. \\
		$\beta$ & Learnable parameter controlling the magnitude scale of the quadratic filter. \\
		${\tilde \lambda}$ & The effective cutoff proxy, defined as ${\tilde \lambda} = 2 - \alpha$. \\
		$\mathbf{Z}_k$ & Intermediate feature representation at propagation step $k$, e.g., $\mathbf{Z}_2 = \mathbf{L}^2X$. \\
		$H_k$ & The output representation of the $k$-th filter channel (low, mid, high, or all). \\
		$H_{\text{concat}}$ & Concatenated representations from all filter channels. \\
		$W_{\text{gate}}, b_{\text{gate}}$ & Learnable weight matrix and bias vector for the gating mechanism. \\
		$S$ & Node-adaptive gating scores (attention weights) for channel fusion. \\
		$\Theta_k$ & Learnable linear transformation matrix for channel $k$. \\
		$\sigma(\cdot)$ & Activation function (e.g., ReLU). \\
		$S_{\mathrm{VN}}(P)$ & Von Neumann entropy of the normalized power response matrix $P$. \\
		\bottomrule
	\end{tabular}
	\label{app: tabNotation}
\end{table*}

\section{More Discussion}\label{app:discussion} 
In this section, we first present a technical comparison between DCQ-GNN and representative adaptive polynomial spectral methods, including ChebyNet, BernNet, and JacobiConv. We then analyze the sensitivity of the proposed model and discuss its computational complexity.
\subsection{Comparison with Polynomial Spectral GNNs}
Polynomial spectral GNNs such as ChebyNet, BernNet, and JacobiConv learn spectral filters through basis expansions. Specifically, these models parameterize the filter as a linear combination of predefined polynomial bases, where the learned coefficients implicitly determine the spectral response.

In contrast, the proposed DCQ-GNN does not rely on polynomial basis expansions to approximate arbitrary filters. Instead, it explicitly constructs a filter bank characterized by curvature polarity. Each filter is designed to satisfy structural second-order derivative constraints (convex, concave, or mid-pass), leading to interpretable and stable spectral responses. This design shifts the focus from polynomial degree to curvature structure, enabling expressive filtering with a small number of parameters.

\subsection{Sensitivity Analysis}
We analyze the stability by considering sensitivity to perturbations in both the spatial and spectral domains.

{\bf Spatial Sensitivity.}
Differentiating the quadratic propagation operator with respect to a perturbed edge weight $w_{uv}$ yields
\begin{equation}
	\nabla_{w_{uv}}h_u^{\mathrm{quad}} = \mathbf{x}_v - \gamma \cdot h_v^{\mathrm{local}},
\end{equation}
where $\gamma$ is a coefficient derived from the second-order filter weight,   $h_v^{\mathrm{local}}$ denotes the local aggregated representation of node $v$ obtained after one propagation step. If the attacker node $v$ is a ``bridge'' to a distinct community (heterophilic noise), $h_v^{\mathrm{local}}$ differs from $\mathbf{x}_v$, and the filter maintains high sensitivity. However, if the perturbation is locally homophilic, the gradient is suppressed.

{\bf Spectral Sensitivity.}
Structural perturbations induce shifts in the Laplacian eigenvalues. We quantify the filter's sensitivity to these shifts via the derivative magnitude $\mathcal{S}_g(\lambda) = |g'(\lambda)|$.
For a linear filter $g_\ell(\lambda) = 1 - \frac{1}{2}\lambda$, the sensitivity is constant: $\mathcal{S}_\ell(\lambda) = 0.5$.
For our convex high-pass filter $g_q(\lambda) = \frac{\beta}{{\tilde \lambda}^2}\lambda^2$, the sensitivity $\mathcal{S}_q(\lambda) = \frac{2\beta}{{\tilde \lambda}^2}\lambda$ vanishes as $\lambda \to 0$. This implies that the quadratic filter exhibits reduced sensitivity near low-frequency components of the spectrum. Conversely, for low-pass components, the quadratic form provides a steeper ``roll-off'' near the cutoff, improving rejection of high-frequency topology noise.

\subsection{Computational Complexity}
The computational cost of DCQ-GNN is primarily dominated by the Sparse-Dense Matrix Multiplications (SpMM) required for spectral propagation. For a graph $G = (\mathcal{V}, \mathcal{E})$ with $N$ nodes, $M$ edges, and input features $X \in \mathbb{R}^{N \times F}$, the complexity for a single layer with a filter bank of size $B$ is analyzed as follows: (1) Spectral Propagation. Each quadratic spectral filter $g(\mathbf{L}) = a_0 I + a_1 \mathbf{L} + a_2 \mathbf{L}^2$ is implemented via recursive neighborhood aggregation operations. The second-order term $\mathbf{L}^2 X$ is computed as $\mathbf{L}(\mathbf{L}X)$, requiring two SpMM operations. Thus, for $B$ filters, the propagation cost is $\mathcal{O}(BMF)$. (2) Adaptive Gating and Fusion. The node-wise attention mechanism involves a linear projection to compute scores and a weighted summation of $B$ filter outputs. This contributes a complexity of $\mathcal{O}(BNF)$. The total complexity per layer is $\mathcal{O}(B(MF + NF))$, which remains linear with respect to the number of edges $M$. This ensures that DCQ-GNN maintains the efficiency of low-order spatial GNNs while offering enhanced spectral selectivity.

\subsection{Receptive Field Expansion and Over-squashing}
\label{app:oversquashing}
\textbf{Over-squashing.}
Over-squashing~\cite{alon2021bottleneck} refers to the compression of long-range signals through narrow graph bottlenecks as propagation depth increases. In DCQ-GNN, each quadratic layer uses second-order Laplacian filtering, where $g(\mathbf{L})=a_0I+a_1\mathbf{L}+a_2\mathbf{L}^2$ can be equivalently expressed as a weighted combination of $I$, $\hat{A}$, and $\hat{A}^2$ since $\mathbf{L}=I-\hat{A}$. Thus, a single layer aggregates identity, one-hop, and two-hop components without stacking two sequential first-order layers. For a graph with average degree $d$, the two-hop term can access up to $\mathcal{O}(d^2)$ local routes, increasing the number of parallel aggregation paths before depth is increased. This second-order expansion provides a wider local receptive field than a purely first-order propagation step while keeping the per-layer propagation order fixed. Consequently, increasing the filter bank depth $T$ expands the receptive field to at most $2T$ hops without relying on a depth-$2T$ stack of first-order transformations. This mechanism is consistent with the stable accuracy observed from $T=1$ to $T=5$ in Table~\ref{tab:complex_compa}.

\section{Extended Theoretical Results}\label{app:theory}

\subsection{Proof of Main Results}\label{app:theory-proof}
This section first presents the proofs of Theorems~\ref{thm:main-diri-energy-rate} and~\ref{thm:main-entropyMajor}, together with Proposition~\ref{prop:main-sturcInfer}, and then provides additional results on quadratic spectral filter decomposition.
\begin{proof}[\bf Proof of Theorem \ref{thm:main-diri-energy-rate}]
	Let $f_\ell(a) = \sup_{\lambda \in [\lambda^*, 2]} \lambda(1 - a\lambda)^2$, and  $a^*$ be the optimal coefficient that minimizes this maximum error, defining the optimal linear filter
	$$g_\ell^*(\lambda) = 1 - a^*\lambda.$$
	Because $g_\ell^*(\lambda)$ is monotonic and linear, the minimax criterion dictates that its maximum weighted energy, denoted as $M_L$, must be balanced at the boundaries of the transition band:
	$$M_L = \lambda^*(1 - a^*\lambda^*)^2 = 2(1 - 2a^*)^2.$$
	To balance these bounds, the filter must cross zero within the interval, implying $g_\ell^*(\lambda^*) > 0$ and $g_\ell^*(2) < 0$.
	We now construct a quadratic filter $g_q(\lambda) \in \mathcal{G}_Q$ that strictly reduces this maximum energy. Define a two-parameter perturbation:
	$$g_q(\lambda) = 1 - (a^* + \delta)\lambda + b\lambda^2, $$
	where $\delta > 0$ and $b > 0$ are arbitrarily small perturbation constants. We analyze the filter's energy at the critical boundaries: 
	
	(1)  At the upper boundary $\lambda = 2$: We require the absolute magnitude of the filter to be strictly reduced: 
	$$|g_q(2)| < |g_\ell^*(2)|.$$
	Since $g_\ell^*(2) = 1 - 2a^* < 0$, we substitute our perturbed filter and reverse the signs:
	$$-(1 - 2a^* - 2\delta + 4b) < -(1 - 2a^*)$$
	This simplifies to the condition $4b - 2\delta > 0$, or equivalently:
	$$b > \frac{\delta}{2}.$$
	
	(2) At the lower boundary $\lambda = \lambda^*$: Similarly, we require $|g_q(\lambda^*)| < |g_\ell^*(\lambda^*)|$. Since $g_\ell^*(\lambda^*) > 0$:
	$$1 - (a^* + \delta)\lambda^* + a(\lambda^*)^2 < 1 - a^*\lambda^*.$$
	This simplifies to:
	$$b(\lambda^*)^2 - \delta\lambda^* < 0 \implies b < \frac{\delta}{\lambda^*}.$$
	Because $\lambda^* \in (0, 2)$, we  have $\frac{1}{2} < \frac{1}{\lambda^*}$. Therefore, there exists a non-empty feasible region for our perturbation parameters. We can cleanly choose a sufficiently small $\delta > 0$ and set $b$ such that:
	$$\frac{\delta}{2} < b < \frac{\delta}{\lambda^*}.$$
	By satisfying this condition, we guarantee that the weighted Dirichlet energy is strictly reduced at both boundaries:
	$$\lambda^*|g_q(\lambda^*)|^2 < M_L \quad \text{and} \quad 2|g_q(2)|^2 < M_L.$$
	Because the polynomials are continuous and evaluated on a compact interval, a sufficiently small choice of $\delta$ ensures that no interior point $\lambda \in (\lambda^*, 2)$ exceeds the new strictly lowered bounds. 
	
	Furthermore, because the perturbation is arbitrarily small and $g_q(0) = 1$, the filter naturally respects the stability constraint $\sup_{\lambda \in [0, 2]} |g_q(\lambda)|\le 1$.
	Consequently, we achieve:
	$$\sup_{\lambda\in[\lambda^*,2]} \lambda|g_q(\lambda)|^2 < M_L.$$
	This confirms the existence of a strictly superior quadratic filter, completing the proof. $\square$
	
\end{proof}

\begin{proof} [\bf Proof of Theorem \ref{thm:main-entropyMajor}]
	Let the eigenvalues of the normalized Laplacian be ordered $0 = \lambda_1 \leq \lambda_2 \leq \dots \leq \lambda_N \leq 2$. The normalized power spectral components are:
	$$p_i = \frac{(\lambda_{\max}-\lambda_i)^2}{\sum_{k=1}^N (\lambda_{\max}-\lambda_k)^2}, \quad q_i = \frac{(\lambda_{\max}-\lambda_i)^4}{\sum_{k=1}^N (\lambda_{\max}-\lambda_k)^4}.$$
	Define the ratio $r_i = \frac{q_i}{p_i} = C \cdot (\lambda_{\max}-\lambda_i)^2$, where $C$ is the ratio of the normalization constants. Since $\lambda_{\max}-\lambda_i$ is non-increasing in $i$, the sequence $\{r_i\}_{i=1}^N$ is non-increasing.
	
	Notice that $\sum p_i = \sum q_i = 1$, the difference $d_i = q_i - p_i$ changes sign at least once. Because $r_i$ is non-increasing, there exists an index $k^*$ such that:
	\begin{itemize}
		\item $q_i \geq p_i$ for $i \leq k^*$
		\item $q_i < p_i$ for $i > k^*$.
	\end{itemize}
	This single-crossing property implies that the mass of the distribution $\mathbf{q}$ is more heavily weighted toward the smaller indices (lower frequencies) compared to $\mathbf{p}$.
	
	Let $P_k = \sum_{i=1}^k p_i$ and $Q_k = \sum_{i=1}^k q_i$ be the cumulative sums. From the single-crossing property:
	
	For $k \leq k^*$, $Q_k - P_k = \sum_{i=1}^k (q_i - p_i) \geq 0$ because all terms are non-negative.
	
	For $k > k^*$, $Q_k - P_k = (Q_N - P_N) - \sum_{i=k+1}^N (q_i - p_i)$. 
	
	Since $Q_N = P_N = 1$ and $(q_i - p_i) < 0$ for $i > k^*$, the sum $\sum_{i=k+1}^N (q_i - p_i)$ is negative. Thus, $Q_k - P_k \geq 0$ for all $k$.
	By definition, this satisfies the condition for majorization: $\mathbf{q} \succ \mathbf{p} $.
	
	Since ${\bf L}$ (thus $P$) is diagonal, then the von Neumann entropy $S_{\mathrm{VN}}(P) = -\sum \eta_i \ln \eta_i$ is a strictly Schur-concave function. By the property of Schur-concavity:
	$$\mathbf{q}\succ \mathbf{p}\implies S_{\mathrm{VN}}(\mathbf{q}) \leq S_{\mathrm{VN}}(\mathbf{p}).$$
	The inequality is strict unless $\mathcal {P} $ is a permutation of $\mathcal {Q} $, which, given the monotonicity of the filters, occurs only if the relevant eigenvalues are degenerate. $\square$
\end{proof}
\begin{proof}[\bf Proof of Proposition \ref{prop:main-sturcInfer}]
	Consider the simplified concave quadratic filter $g(\mathbf{L}) = 2\mathbf{L} - \mathbf{L}^2$. In terms of the normalized adjacency $\hat{A} = I - \mathbf{L}$, we have:
	$$g(\mathbf{L}) = 2(I- \hat{A}) - (I - \hat{A})^2 = I - \hat{A}^2$$
	To formalize the structural interference cancellation, consider the output representation for a target node $u$, given by $h_u = x_u - [\hat{A}^2 X]_u$. Suppose an adversarial edge $(u, v)$ is injected into the graph to corrupt node $u$.
	In a standard first-order GCN (where $h_{GCN} = \hat{A}X$), the direct feature pollution from the adversarial node $v$ onto $u$ is given by:
	$$\Delta_{GCN} = \hat{A}_{uv}x_v = \frac{1}{\sqrt{d_u d_v}}x_v.$$
	In DCQ-GNN, the application of $\hat{A}^2$ computes a two-hop random walk. The structural interference reaching $u$ along the adversarial path $v$ is: 
	$$[\hat{A}^2 X]_{u \leftarrow v} = \sum_{t \in \mathcal{N}(v)} \hat{A}_{uv}\hat{A}_{vt}x_t = \frac{1}{\sqrt{d_u d_v}} \sum_{t \in \mathcal{N}(v)} \frac{1}{\sqrt{d_v d_t}} x_t.$$
	Assume the adversarial node $v$ is embedded in a locally homophilic neighborhood such that its neighbors' features are $\epsilon$-consistent with $x_v$. Formally, let $x_t = x_v + \xi_t$, where $\|\xi_t\| \le \epsilon$ for all $t \in \mathcal{N}(v)$. Assuming bounded degree variation within the local neighborhood (i.e., $d_t \in [c_1 d_v, c_2 d_v]$ for constants $c_1,c_2>0$), the adversarial contribution becomes: 
	$$[\hat{A}^2 X]_{u \leftarrow v} \approx \frac{1}{\sqrt{d_u d_v}} \sum_{t \in \mathcal{N}(v)} \frac{1}{d_v} (x_v + \xi_t) = \frac{|\mathcal{N}(v)|}{d_v\sqrt{d_u d_v}}x_v + Err,$$
	where the error term $Err = \frac{1}{d_v\sqrt{d_u d_v}} \sum_{t \in \mathcal{N}(v)} \xi_t$.
	Since $|\mathcal{N}(v)| = d_v$, the primary feature injection term evaluates exactly to $\frac{1}{\sqrt{d_u d_v}}x_v$, which is identical to the first-order pollution $\Delta_{GCN}$. Because the filter applies $I - \hat{A}^2$, this adversarial signal is explicitly subtracted from the node's representation.
	Furthermore, by the triangle inequality, the residual error is bounded by:
	$$\|Err\| \le \frac{1}{d_v\sqrt{d_u d_v}} \sum_{t \in \mathcal{N}(v)} \|\xi_t\| \le \frac{\epsilon}{\sqrt{d_u d_v}}.$$
	Thus, as local homophily around the attacker increases ($\epsilon \to 0$), the quadratic filter effectively cancels the adversarial injection. $\square$
\end{proof}
\subsection{Stability Analysis}
This section establishes the structural robustness of DCQ-GNN by characterizing its response under Laplacian perturbations.
\begin{theorem}[Spectral Stability Under Laplacian Perturbation]\label{app:thm-strucprtur}
	Let  $\mathbf L$ be the normalized graph Laplacian of a graph $G$, and let $\tilde{\mathbf L} = \mathbf L + \Delta$ denote a perturbed Laplacian where $\|\Delta\|_2 \le \epsilon$.  Consider a quadratic spectral filter 
	$$g(\mathbf L) = a_0 I + a_1 \mathbf L + a_2 {\mathbf L}^2 .$$
	Then the perturbation in the filtered signal satisfies
	$$\|g(\tilde{\mathbf L})x - g(\mathbf L)x\|_2\le \epsilon\left(|a_1| + 4|a_2| + |a_2|\epsilon\right)\|x\|_2 .$$
	Consequently, the quadratic spectral filtering operator exhibits Lipschitz stability with respect to structural perturbations of the graph topology.
\end{theorem}
\begin{proof}
	Let
	$$g(\mathbf L) = a_0 I + a_1 \mathbf L + a_2 {\mathbf L}^2 .$$
	Evaluating the filter on the perturbed operator yields:
	$$
	g(\tilde{\mathbf L}) = a_0 I + a_1 (\mathbf L+\Delta) + a_2 (\mathbf L+\Delta)^2.
	$$
	Expanding the second-order term, we have
	$$(\mathbf L+\Delta)^2 ={\mathbf L}^2 + \mathbf L\Delta + \Delta \mathbf L + \Delta^2 .$$
	The operator difference is consequently expressed as:
	$$g(\tilde{\mathbf L}) - g(\mathbf L)= a_1 \Delta +a_2 (\mathbf L\Delta + \Delta {\mathbf L})+a_2 \Delta^2 .$$
	Taking the spectral norm and using the triangle inequality, we have
	$$\|g(\tilde{\mathbf L}) - g(\mathbf L)\|_2\le |a_1| \|\Delta\|_2+
	|a_2|(\|\mathbf L\Delta\|_2 + \|\Delta \mathbf L\|_2)+|a_2| \|\Delta^2\|_2 .$$
	By the submultiplicativity of the operator norm and noting that $\|\mathbf{L}\|_2 \le 2$ for normalized Laplacians, the individual terms are bounded as $\|\mathbf{L}\Delta\|_2 \le 2\epsilon$ and $\|\Delta^2\|_2 \le \epsilon^2$. Substituting these into the inequality yields:
	$$\|g(\tilde{\mathbf L}) - g(\mathbf L)\|_2\le \epsilon \left(|a_1|+2|a_2|\|\mathbf L\|_2+|a_2|\epsilon\right).$$
	Finally, for any signal $x \in \mathbb{R}^n$, the perturbation error satisfies:
	\begin{align*}
		\|g(\tilde{\mathbf L})x - g(\mathbf L)x\|_2
		&\le\|g(\tilde{\mathbf L}) - g(\mathbf L)\|_2 \|x\|_2 \\
		&\le\epsilon\left(|a_1|+2|a_2|\|\mathbf L\|_2+|a_2|\epsilon\right)\|x\|_2\\
		&\le \epsilon\left(|a_1|+4|a_2|+|a_2|\epsilon\right)\|x\|_2.
	\end{align*}
	This completes the proof, confirming that the quadratic filter exhibits bounded output variation under bounded topological perturbations.
	$\square$.
\end{proof}
\subsection{Spectral Selectivity of  Quadratic Filters}
\begin{theorem}[Enhanced Spectral Resolution of Quadratic Formulations] 
	Let $\mathbf{L}$ be the normalized graph Laplacian with eigenvalues $\lambda \in [0,2]$. Define a linear filter response as $g_{\ell}(\lambda) = a_0+ a_1\lambda$ and a quadratic counterpart as $g_{q}(\lambda) = a_0 + a_1\lambda + a_2\lambda^2$. For any non-zero $a_2$, there exists a non-empty spectral interval $[\lambda_1, \lambda_2] \subseteq [0,2]$ such that:
	$$|g_{q}(\lambda_1)-g_{q}(\lambda_2)| > |g_{\ell}(\lambda_1)-g_{\ell}(\lambda_2)|.$$
	Thus, quadratic filters can achieve strictly higher spectral selectivity under compatible coefficient polarity conditions, enabling sharper discrimination of high-frequency components relative to the linear regime. 
\end{theorem}
This result shows that the second-order term can enlarge spectral separation on suitable intervals without necessarily amplifying perturbation-sensitive components.
\begin{proof} 
	Consider the spectral gap between two frequencies $\lambda_1, \lambda_2$. For the linear filter, the gap is $\Delta g_{\ell} = a_1(\lambda_1-\lambda_2)$. For the quadratic case, the gap expands to:
	$$\Delta g_{q} = a_1(\lambda_1-\lambda_2) + a_2(\lambda_1^2-\lambda_2^2) = (\lambda_1-\lambda_2)\left[a_1 + a_2(\lambda_1+\lambda_2)\right].$$
	For any pair of non-trivial frequencies such that $\text{sgn}(a_1) = \text{sgn}(a_2(\lambda_1+\lambda_2))$, the absolute difference satisfies:
	$$|g_{q}(\lambda_1)-g_{q}(\lambda_2)| = |(\lambda_1-\lambda_2)| \cdot |a_1 + a_2(\lambda_1+\lambda_2)| > |a_1(\lambda_1-\lambda_2)|.$$
	This amplification of the spectral distance confirms that the quadratic term $a_2\lambda^2$ increases the filter's sensitivity to frequency variations, particularly in the high-frequency tail ($\lambda \approx 2$) where $a_2(\lambda_1+\lambda_2)$ is maximized. This completes the proof. 
	$\square$ 
\end{proof}

\subsection{Curvature-Based Spectral Expressivity}\label{app: perfor-exp}
This section investigates the expressive capacity of convex--concave combinations relative to uniform-polarity banks. Based on the curvature polarity ($\mathcal{CP}$) introduced in Definition \ref{def:main-curPol}, we demonstrate that opposing polarities are a necessary condition for spectral flexibility within low-order constraints.
Consider two convex filters $g_1(\lambda)$ and $g_2(\lambda)$ with $\mathcal{CP}=1$, combined as $H(\lambda) = w_1 g_1(\lambda) + w_2 g_2(\lambda)$, where $w_1, w_2 \in [0,1]$. Since the second derivative of a conical combination of convex functions remains non-negative ($H''(\lambda) \geq 0$), the resulting filter is strictly constrained to the convex functional space. Consequently, its response is restricted to monotonic or U-shaped profiles, precluding the modeling of localized mid-frequency structures common in heterophilic graphs.
In contrast, the proposed convex--concave combination $H(\lambda) = w_1 g_{\text{low}}^{\text{cvx}}(\lambda) + w_2 g_{\text{high}}^{\text{ccv}}(\lambda)$ yields a second derivative:
$$H''(\lambda) = \frac{2 (w_1 - w_2)\beta}{{\tilde \lambda}^2}.$$
Analytic derivation reveals that the curvature polarity of $H(\lambda)$ is continuously steerable across the set $\{-1, 0, 1\}$ by modulating the gating weights $w_1, w_2$. Unlike uniform-polarity filters, this configuration facilitates a band-pass response when the first and second derivatives vanish at critical spectral points—a behavior mathematically unattainable via convex combinations alone. Consequently, the convex--concave pair serves as a flexible generator for second-order spectral responses, capturing the full spectrum of homophilic and heterophilic signals within a compact polynomial family.

\begin{figure}[ht]
	\centering
	\includegraphics[width=0.96\linewidth]{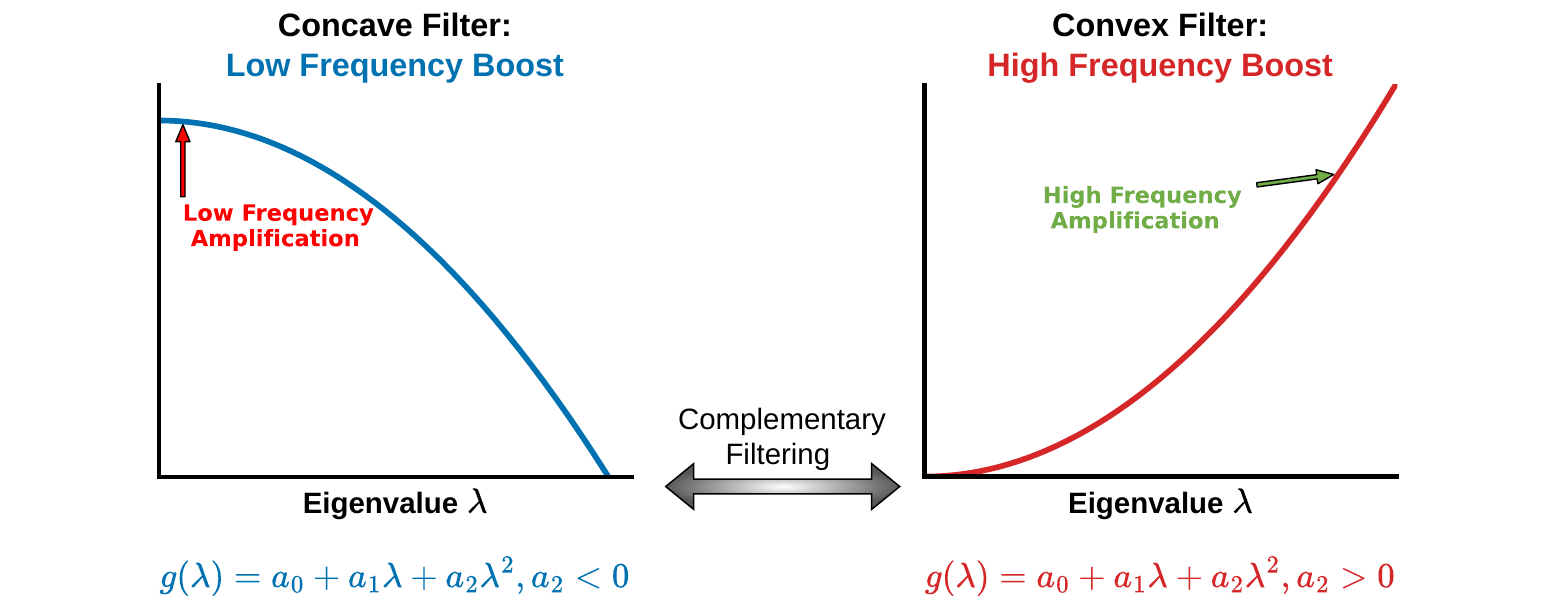}
	\caption{Spectral curvature analysis under different polarity coefficients. The sign of the quadratic coefficient ($a_2$) determines the curvature of the spectral response. Pairing convex and concave filters enables complementary low- and high-frequency modeling within a single layer.}
	\label{fig: app-curvature}
\end{figure}
The following result formally characterizes this spectral bias through the lens of second-order derivatives (see Fig.~\ref{fig: app-curvature}).
\begin{theorem}[Spectral Curvature Theorem]
	\label{thm:app-spec-curv}
	Let ${\mathbf L}$ be the normalized graph Laplacian with eigenvalues $\lambda \in [0,2]$. Consider a quadratic spectral filter $g(\lambda)=a_0+a_1\lambda+a_2\lambda^2$ with curvature $g''(\lambda)=2a_2$. The spectral bias is determined as follows:
	\begin{itemize}
		\item If $a_2>0$ (convexity), the filter exhibits high-frequency amplification relative to the low-frequency regime.
		\item If $a_2<0$ (concavity), the filter prioritizes low-frequency components by attenuating the high-frequency gain.
	\end{itemize}
	Consequently, a convex--concave pair constitutes a complementary spectral decomposition for the simultaneous capture of multi-scale graph signals. 
\end{theorem}
Here, $g(\lambda)$ denotes the final spectral response defined in Eqs. \eqref{cvx-low}–\eqref{ccv-high}. The curvature polarity is positive for convex filters and negative for concave filters, which determines how aggressively the filter attenuates frequencies across the spectrum. In our framework, the coefficients $a_i$ ($i \in \{0,1,2\}$) are parameterized by the learnable cutoff $\alpha$ and scale 
$\beta$, allowing the spectral selectivity to remain adaptive. Theorem \ref{thm:app-spec-curv} further establishes that the curvature coefficient governs the spectral bias: convex filters tend to emphasize high-frequency components, whereas concave filters favor low-frequency components, enabling complementary frequency decomposition.
\begin{proof}
	Let $\lambda_1 < \lambda_2$ be two arbitrary eigenvalues. The relative amplification is given by $\Delta g = g(\lambda_2)-g(\lambda_1)$. Substituting the quadratic form:
	$$\Delta g=a_1(\lambda_2-\lambda_1)+a_2(\lambda_2^2-\lambda_1^2) = (\lambda_2-\lambda_1)\left[a_1+a_2(\lambda_1+\lambda_2)\right].$$
	Since $\lambda_1+\lambda_2 > 0$  for all non-trivial eigenvalues, the sign of the second-order coefficient $a_2$ governs the monotonicity of the gain relative to the frequency magnitude. Specifically, $a_2 > 0$ induces a positive curvature that scales the gain with $\lambda$, whereas $a_2 < 0$ imposes a concave penalty on higher frequencies. This proves that curvature polarity is the mechanical driver of spectral bias.
	$\square$
\end{proof}
In fact, any second-order spectral filter can be decomposed into the combination of convex and concave filters, which is presented in the following theorem.
\begin{theorem}[Convex--Concave Decomposition of Quadratic Filters]
	\label{thm: app-convex_concave_spanning}
	Let $\mathcal{Q}_2 = \{ q(\lambda) = a_0 + a_1\lambda + a_2\lambda^2: a_0, a_1, a_2 \in \mathbb{R} \}$ denote the space of second-order polynomials on the interval $[0,2]$ Define the polar sub-spaces:
	$$\mathcal{Q}_{cvx} = \{ q \in \mathcal{Q}_2 : a_2 > 0 \},\qquad \mathcal{Q}_{ccv} = \{ q \in \mathcal{Q}_2 : a_2 < 0 \}.
	$$
	Every quadratic filter $q \in \mathcal{Q}_2$ admits a decomposition $q = q_{cvx} + q_{ccv}$, where $q_{cvx} \in \mathcal{Q}_{cvx}$ and $q_{ccv} \in \mathcal{Q}_{ccv}$. This establishes that the space of quadratic filters $\mathcal{Q}_2$ is the sum of the convex and concave quadratic cones, $\mathcal{Q}_2 = \mathcal{Q}_{cvx} + \mathcal{Q}_{ccv}$.
\end{theorem} 
This result establishes that the proposed convex--concave architecture is not merely an intuitive heuristic, but a two-branch convex--concave decomposition for second-order spectral filters.
\begin{proof} 
	For any $q(\lambda) = a_0 + a_1\lambda + a_2\lambda^2$, select any $\delta > |a_2|$ and define: $q_{cvx}(\lambda) = \delta \lambda^2$ and $q_{ccv}(\lambda) = a_0 + a_1\lambda  + (a_2-\delta)\lambda^2$. By construction, $\delta > 0$ ensures $q_{cvx} \in \mathcal{Q}_{cvx}$, while $a_2-\delta < 0$ guarantees $q_{ccv} \in \mathcal{Q}_{ccv}$. The identity $q = q_{cvx} + q_{ccv}$ holds trivially. $\square$
\end{proof} 
\paragraph{\bf Theoretical Implication.}
The presented analysis establishes three properties of DCQ-GNN:
\begin{itemize}
	\item  quadratic propagation increases spectral selectivity relative to linear filters;
	
	\item curvature polarity induces complementary frequency responses;
	
	\item the operator remains stable under bounded Laplacian perturbations.
\end{itemize}
Together, these results explain why DCQ-GNN can simultaneously capture low- and high-frequency graph signals while maintaining structural robustness.
\section{Algorithmic Details}\label{app:algorithms}

\begin{algorithm}[!htp]
	\renewcommand{\algorithmicrequire}{\textbf{Input:}}
	\renewcommand\algorithmicensure {\textbf{Output:} }
	\caption{Overall pipeline of the DCQ-GNN framework.}
	\label{app: alg:forward}
	\begin{algorithmic}[1]
		\REQUIRE Graph Laplacian $\mathbf{L} \in \mathbb{R}^{N \times N}$, 
		Node features $X \in \mathbb{R}^{N \times F}$
		\ENSURE Output embedding $H_{\text{out}} \in \mathbb{R}^{N \times F'}$
		
		\STATE \textbf{Step 1: Second-Order Recursive Propagation}
		\STATE $\mathbf{Z}_0 \leftarrow X$
		\STATE $\mathbf{Z}_1 \leftarrow \mathbf{L}\mathbf{Z}_0$ \COMMENT{One-hop propagation}
		\STATE $\mathbf{Z}_2 \leftarrow \mathbf{L}\mathbf{Z}_1$ \COMMENT{Two-hop propagation}
		
		\STATE \textbf{Step 2: Construction of Quadratic Filter Bank}
		\STATE $\tilde{\lambda} \leftarrow 2 - \alpha$ \COMMENT{Cutoff proxy}
		\STATE $c \leftarrow \beta / \tilde{\lambda}^2$ \COMMENT{Magnitude normalization}
		
		\IF{low-pass curvature is convex}
		\STATE $H_{\text{low}} \leftarrow 
		c \big(\mathbf{Z}_2 - 2\tilde{\lambda}\mathbf{Z}_1 
		+ \tilde{\lambda}^2 \mathbf{Z}_0 \big)$
		\ELSE
		\STATE $H_{\text{low}} \leftarrow 
		-c \big(\mathbf{Z}_2 - \tilde{\lambda}^2 \mathbf{Z}_0 \big)$
		\ENDIF
		
		\STATE $H_{\text{mid}} \leftarrow 
		-4c \big(\mathbf{Z}_2 - \tilde{\lambda}\mathbf{Z}_1 \big)$
		
		\IF{high-pass curvature is convex}
		\STATE $H_{\text{high}} \leftarrow c\,\mathbf{Z}_2$
		\ELSE
		\STATE $H_{\text{high}} \leftarrow 
		-c \big(\mathbf{Z}_2 - 2\tilde{\lambda}\mathbf{Z}_1 \big)$
		\ENDIF
		
		\STATE $H_{\text{all}} \leftarrow \mathbf{Z}_0$ \COMMENT{All-pass residual anchor}
		
		\STATE \textbf{Step 3: Channel-wise Linear Projection}
		\FOR{$k \in \{\text{low}, \text{mid}, \text{high}, \text{all}\}$}
		\STATE $\tilde{H}_k \leftarrow H_k \mathbf{\Theta}_k$
		\ENDFOR
		
		\STATE \textbf{Step 4: Node-Adaptive Gated Fusion}
		\STATE $H_{\text{concat}} \leftarrow 
		[H_{\text{low}} \| H_{\text{mid}} 
		\| H_{\text{high}} \| H_{\text{all}}]$
		\STATE $S \leftarrow 
		\mathrm{softmax}(H_{\text{concat}}W_{\text{gate}} 
		+ \mathbf{b}_{\text{gate}})$
		\STATE $H_{\text{out}} \leftarrow 
		\sum_k S_k \odot \tilde{H}_k$
		\STATE $H_{\text{out}} \leftarrow 
		\sigma(H_{\text{out}})$
		
		\RETURN $H_{\text{out}}$
	\end{algorithmic}
\end{algorithm}
The overall pipeline of DCQ-GNN is summarized in Algorithm~\ref{app: alg:forward}. The quadratic filter responses are computed via second-order recursive Laplacian propagation, thereby avoiding explicit polynomial expansion or eigendecomposition.  Each propagation step consists of a sparse–dense matrix multiplication. 

\section{Detailed Experiments}\label{app:experiments} 
This appendix provides supplementary experimental details, including dataset statistics (Table~\ref{tab:app-dataset_statistics}), baseline descriptions, and hyperparameter configurations. In addition, we extend the empirical study to further analyze robustness under adversarial perturbations, spectral adaptivity across structural regimes, and performance under class-balanced training protocols. Beyond the primary results reported in the main text, we address the following additional research questions:

\begin{itemize}
	\item \textbf{RQ2 Extension (Robustness Analysis):} Does the proposed structural interference correction mechanism remain effective under adversarial topology attacks on strongly homophilic graphs (e.g., \textit{PubMed})?
	
	\item \textbf{RQ5 (Spectral Selectivity and Adaptivity):} How do the learned spectral responses of the DCQ-GNN filter bank adapt across homophilic and heterophilic graph regimes compared to standard polynomial filters?
	
	\item \textbf{RQ6 (Evaluation under Dense Splits):} How does DCQ-GNN perform relative to recent spectral baselines when evaluated under a class-balanced (dense split) training protocol?
\end{itemize}
\subsection{Datasets}\label{app:experiments:data}
\begin{table}[ht] 
	\centering 
	\caption{Summary of dataset statistics and structural properties. Datasets are grouped into homophilic and heterophilic regimes. Reported metrics include the number of nodes $|\mathcal{V}|$, edges $|\mathcal{E}|$, feature dimension $F$, number of classes, assortativity, and three homophily measures ($h_{\text{node}}$, $h_{\text{edge}}$, $h_{\text{class}}$).}
	\resizebox{0.99\textwidth}{!}{
		\begin{tabular}{llrrrcrrrc}
			\toprule 
			\textbf{Type} &\textbf{Dataset} & $\mathbf{|\mathcal{V}|}$ & $\mathbf{|\mathcal{E}|}$ & $F$ & Classes & $h_{\text{node}}$ & $h_{\text{edge}}$ & $h_{\text{class}}$ & Assortativity \\ 
			\midrule 
			&\textbf{Photo} & 7,650 & 238,162 & 745 & 8 & 0.8365 & 0.8272 & 0.7722 & -0.0449 \\ 
			&\textbf{PubMed} & 19,717 & 88,648 & 500 & 3 & 0.7924 & 0.824 & 0.6641 & -0.0436 \\ 
			Homophilic&\textbf{Computers} & 13,752 & 491,722 & 767 & 10 & 0.7853 & 0.7772 & 0.7002 & -0.0565 \\ 
			&\textbf{CiteSeer} & 3,327 & 9,104 & 3,703 & 6 & 0.7062 & 0.7355 & 0.6267 & 0.0484 \\ 
			&\textbf{WikiCS} & 11,701 & 431,726 & 300 & 10 & 0.6588 & 0.6547 & 0.5681 & -0.0653 \\ 
			\hline 
			&\textbf{Wisconsin} & 251 & 515 & 1,703 & 5 & 0.1719 & 0.1961 & 0.0839 & -0.2723 \\ 
			&\textbf{Actor} & 7,600 & 30,019 & 932 & 5 & 0.1586 & 0.2188 & 0.0061 & -0.1111 \\ 
			Heterophilic&\textbf{Chameleon} & 2,277 & 36,101 & 2,325 & 5 & 0.1039 & 0.2350 & 0.0586 & -0.1128 \\ 
			&\textbf{Squirrel} & 5,201 & 217,073 & 2,089 & 5 & 0.0889 & 0.2239 & 0.0430 & 0.3738 \\
			&\textbf{Texas} & 183 & 325 & 1,703 & 5 & 0.0654 & 0.1077 & 0.0000 & -0.3463 \\ 
			\bottomrule 
	\end{tabular} } 
	\label{tab:app-dataset_statistics}
\end{table}
We evaluate the proposed model on $10$ benchmark datasets spanning a broad range of structural properties. Consistent with the main text, datasets are categorized according to homophily level: (i) homophilic graphs, where edges predominantly connect nodes of the same class, and (ii) heterophilic graphs, where edges frequently connect nodes of different classes.
\paragraph{Homophilic Datasets.}
\begin{itemize}
	\item \textbf{CiteSeer \& PubMed}~\cite{yang2016revisiting}: Standard citation networks evaluated using the Planetoid experimental protocol. Nodes represent scientific documents linked by citation edges, with features represented as bag-of-words vectors. The task entails multi-class node classification across distinct research topics.
	
	\item \textbf{Amazon Computers \& Photo}~\cite{shchur2018pitfalls}: Subgraphs extracted from the Amazon co-purchase network. Nodes correspond to products, while edges denote frequent co-purchase relations. Node features are encoded from review text to facilitate product category classification.
	
	\item \textbf{WikiCS}~\cite{mernyei2020wiki}: A semi-supervised benchmark comprising Wikipedia articles related to computer science. Nodes represent articles connected by hyperlinks. The dataset is characterized by multiple predefined training splits to ensure evaluation stability in node classification.
\end{itemize}

\paragraph{Heterophilic Datasets.}

\begin{itemize}
	\item \textbf{WebKB (Texas, Wisconsin)}~\cite{pei2020geom}: Hyperlink networks of university computer science departments. Nodes represent webpages and edges denote hyperlinks. Features are bag-of-words representations. The task is page classification into categories such as faculty, student, and course.
	
	\item \textbf{WikipediaNetwork (Chameleon, Squirrel)}~\cite{rozemberczki2021multi}: Page-to-page hyperlink graphs from Wikipedia. Node features are derived from textual content. The target variable is discretized page traffic.
	
	\item \textbf{Actor}~\cite{pei2020geom}: The actor-induced subgraph of a film collaboration network. Nodes represent actors and edges capture co-occurrence relationships. The task is categorical node classification.
\end{itemize}
These datasets collectively enable evaluation under varying spectral regimes, including low-frequency-dominant (homophilic) and high-frequency-dominant (heterophilic) settings.

\subsection{Baselines}

We compare DCQ-GNN against a diverse set of spatial and spectral graph neural networks:

\begin{itemize}
	\item \textbf{GCN}~\cite{kipf2017semi}: A first-order low-pass spectral approximation with propagation rule $g(\lambda) = 1 - \lambda$.
	
	\item \textbf{GCNII}~\cite{chen2020simple}: Incorporates initial residual connections and identity mapping to alleviate over-smoothing in deep architectures.
	
	\item \textbf{GAT}~\cite{velivckovic2018graph}: Introduces attention-based anisotropic aggregation.
	
	\item \textbf{JK-Net}~\cite{xu2018representation}: Aggregates multi-scale representations through jump connections.
	
	\item \textbf{H2GCN}~\cite{zhu2020beyond}: Explicitly models higher-order neighborhoods and ego-feature separation, designed for heterophilic settings.
	
	\item \textbf{EvenNet}~\cite{lei2022evennet}: Employs even-degree polynomial filters to achieve spectral symmetry.
	
	\item \textbf{FAGCN}~\cite{bo2021beyond}: Integrates low- and high-frequency components through adaptive gating.
	
	\item \textbf{JacobiConv}~\cite{wang2022powerful}: A high-order polynomial spectral model using Jacobi basis functions.
	
	\item \textbf{BernNet}~\cite{he2021bernnet}: Utilizes Bernstein polynomial bases for flexible spectral approximation.
	
	\item \textbf{AutoSGNN}~\cite{mo2025autosgnn}: An automated framework combining large language models and evolutionary search for propagation design.
	
	\item \textbf{LHK-GNN}~\cite{qin2025localized}: Introduces localized heat kernel filtering to address over-smoothing.
\end{itemize}
This selection covers first-order spatial models, attention-based methods, polynomial spectral filters, and kernel-based approaches.

\subsection{Hyperparameter Optimization.}
We employ the Optuna framework for automated hyperparameter tuning, with the primary objective of maximizing validation accuracy. Each dataset is allocated $100$ trials subject to a maximum time budget of $3,600$ seconds. To ensure empirical rigor, all models undergo the same search protocol. In instances where official implementations provided dataset-specific tuned parameters, those configurations are adopted to ensure a competitive and fair baseline performance.

To reduce computational overhead, we employ Optuna's \texttt{MedianPruner} with 10 startup trials, a 50-epoch warm-up period, and evaluation intervals of 10 epochs.

\begin{itemize}
	\item \textbf{General Training Parameters (Shared across all tuned models):}
	\begin{itemize}
		\item \textbf{Learning rate:} Log-uniform distribution $\in [1\text{e-}4, 1\text{e-}2]$.
		\item \textbf{Weight decay:} Log-uniform distribution $\in [1\text{e-}6, 1\text{e-}2]$.
		\item \textbf{Dropout rate:} Uniform distribution $\in [0.0, 0.9]$.
	\end{itemize}
	\item \textbf{DCQ-GNN Architecture-Specific Parameters:}
	\begin{itemize}
		\item \textbf{Number of layers:} Integer $\in [1, 3]$.
		\item \textbf{Low-pass filter curvature:} Categorical $\in \{\text{Convex (cvx)}, \text{Concave (ccv)}\}$, tuned independently for each layer.
		\item \textbf{High-pass filter curvature:} Categorical $\in \{\text{Convex (cvx)}, \text{Concave (ccv)}\}$, tuned independently for each layer.
	\end{itemize}
\end{itemize}

\subsection{Sensitivity to $\alpha$ and $\beta$}
\label{app:parameter_sensitivity}

The parameter $\alpha$ controls the cutoff proxy $\tilde{\lambda}=2-\alpha$, while $\beta$ rescales the quadratic response subject to $|g(\lambda)|\le 1$. Fig.~\ref{fig:app-filters} visualizes how these parameters modulate the filter profiles.

\begin{figure}[ht]
	\centering
	\includegraphics[width=0.89\linewidth]{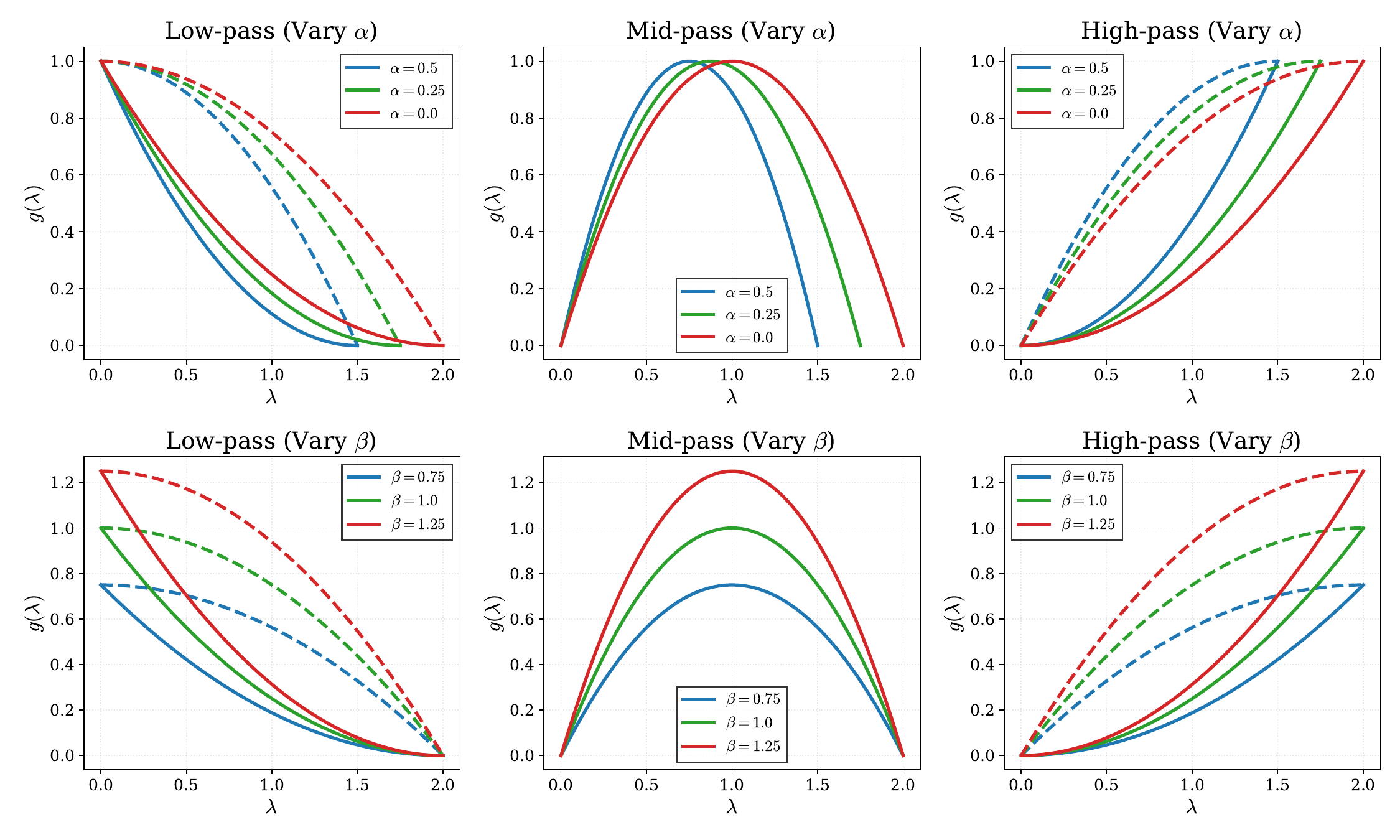}
	\caption{Effect of parameters on spectral response. (Top) $\alpha$ shifts the transition frequency (cutoff), while (Bottom) $\beta$ controls the response magnitude. Convex (solid) and concave (dashed) filters provide complementary spectral sharpness at the same polynomial order.}
	\label{fig:app-filters}
\end{figure}

\subsubsection{Parameter-Specific Weight Decay.}
To avoid unintended regularization of spectral shape parameters, weight decay is applied only to standard network weights. Learnable spectral parameters (cutoff proxies $\alpha$ and magnitude scalars $\beta$), bias terms, and normalization parameters are excluded from weight decay (decay rate set to 0.0). This design ensures that curvature parameters are not artificially suppressed during optimization.

\begin{figure}[htbp]
	\centering
	\includegraphics[width=0.85\linewidth]{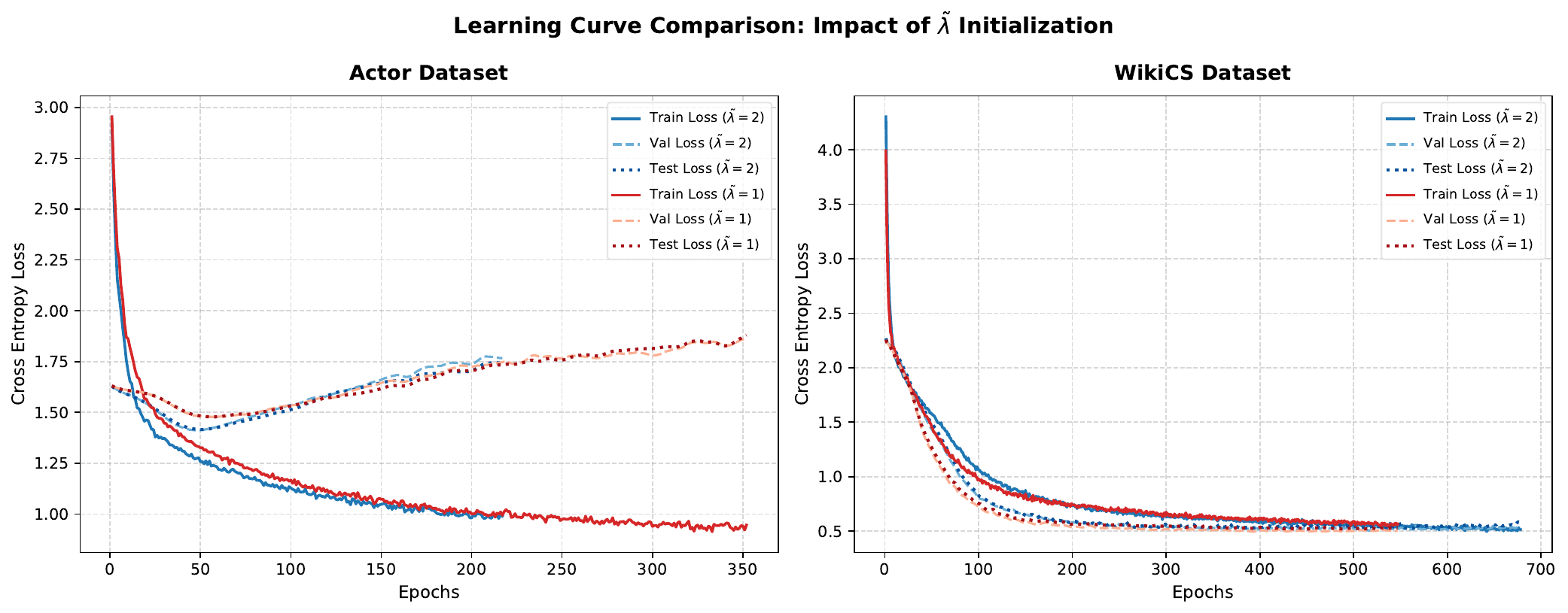} 
	\caption{Training, validation, and test cross-entropy loss trajectories under two cutoff initializations, $\alpha=0$ $(\tilde{\lambda}=2)$ and $\alpha=1$ $(\tilde{\lambda}=1)$, on Actor and WikiCS datasets.}
	\label{fig:learning_curves}
\end{figure}

\paragraph{Convergence Stability Analysis.}
	To examine optimization stability, we track training-loss trajectories under two initialization settings, $\alpha=0$ $(\tilde{\lambda}=2)$ and $\alpha=1$ $(\tilde{\lambda}=1)$, as shown in Figure~\ref{fig:learning_curves}. On WikiCS, the two trajectories are nearly overlapping, decrease smoothly, and stabilize after approximately 400--500 epochs, indicating that optimization is insensitive to the initial cutoff proxy on this homophilic graph. On Actor, the two settings also follow similar trajectories, with the cross-entropy loss decreasing rapidly in the early stage and increasing after about 50 epochs. This pattern suggests that validation loss is not always aligned with final classification accuracy on heterophilic graphs, where confidence calibration and decision-boundary refinement may evolve differently. We therefore select checkpoints by validation accuracy rather than validation loss, which is consistent with the node-classification objective used in our evaluation.

\begin{table}[htbp]
	\centering
	\caption{Node classification accuracy under different initializations of the filter cutoff parameter $\alpha$. Relative changes are computed against the original initialization.}
	\label{tab:init_sensitivity}
	\begin{tabular}{lcc}
		\toprule
		\textbf{Initialization Configuration} & \textbf{WikiCS} & \textbf{Actor} \\
		\midrule
		Original ($\alpha=0$) & $\mathbf{85.50 \pm 0.82}$ & $\mathbf{35.63 \pm 1.39}$ \\
		Altered ($\alpha=1$)  & $84.84 \pm 0.49\ (\downarrow 0.77\%)$ & $33.22 \pm 1.00\ (\downarrow 6.76\%)$ \\
		\bottomrule
	\end{tabular}
\end{table}
\paragraph{Initialization Sensitivity of Accuracy.}
	Although the loss curves remain stable under both initializations, Table~\ref{tab:init_sensitivity} shows that classification accuracy exhibits dataset-dependent sensitivity. On WikiCS, changing the initialization from $\alpha=0$ to $\alpha=1$ causes only a $0.77\%$ relative drop, suggesting that the homophilic smoothing signal is robust to the initial cutoff proxy. On Actor, the same change leads to a larger decrease of $6.76\%$. Since heterophilic graphs rely more heavily on non-low-frequency components, initializing with $\tilde{\lambda}=2$ provides a broader initial spectral range and facilitates the subsequent adaptation of high-frequency channels.

\subsection{Robustness under Adversarial Perturbations (RQ2 Extension)}
\begin{table}[h] 
	\centering 
	\caption{Node classification accuracy under the DICE structural attack on \textbf{PubMed}. $\downarrow \Delta\%$ denotes the relative accuracy drop from the clean graph. The best rounded accuracy at each perturbation ratio is highlighted in bold.}
	\resizebox{\textwidth}{!}{
		\begin{tabular}{lcccc} 
			\toprule 
			& \multicolumn{4}{c}{\textbf{DICE Attack Perturbation Ratio}} \\ \cmidrule(lr){2-5} \textbf{Model} & \textbf{0\% (Clean)} & \textbf{5\%} & \textbf{10\%} & \textbf{20\%} \\ 
			\midrule 
			GCN & $88.53 \pm 0.28$ & $82.77 \pm 0.62 (\downarrow6.51\%)$ & $77.64 \pm 0.59 (\downarrow12.30\%)$ & $67.44 \pm 0.73 (\downarrow23.82\%)$ \\ 
			GAT & $88.12 \pm 0.48$ & $82.82 \pm 0.47 (\downarrow6.01\%)$ & $77.88 \pm 0.68 (\downarrow11.62\%)$ & $69.04 \pm 1.06 (\downarrow21.65\%)$ \\
			FAGCN & $89.21 \pm 0.55$ & $85.78 \pm 0.46 (\downarrow3.85\%)$ & $83.99 \pm 0.84 (\downarrow5.85\%)$ & $78.41 \pm 1.18 (\downarrow12.11\%)$ \\
			BernNet & $89.46 \pm 0.33$ & $87.62 \pm 0.45 (\downarrow2.06\%)$  & $86.17 \pm 0.29 (\downarrow3.68\%)$ & $82.39 \pm 0.23 (\downarrow7.90\%)$\\ 
			JacobiConv & $89.44 \pm 0.10$ & $\mathbf{88.37 \pm 0.18 (\downarrow1.20\%)}$ & $87.36 \pm 0.38 (\downarrow2.32\%)$ & $82.58 \pm 0.22 (\downarrow7.67\%)$ \\  		
			\hline
			\textbf{DCQ-GNN (ours)} & $\mathbf{89.62 \pm 0.49}$ & $ \mathbf{88.37 \pm 0.41 (\downarrow1.39\%)}$  & $\mathbf{87.73 \pm 0.39 (\downarrow2.11\%)}$ & $\mathbf{85.63 \pm 0.68 (\downarrow4.45\%)}$\\
			\bottomrule 
	\end{tabular}}
	\label{tab:app-reverse_dice_performance} 
\end{table}
We further evaluate robustness on PubMed under DICE structural perturbations. PubMed is a homophilic graph ($h=0.79$), so heterophilic edge perturbations directly disrupt the label-consistent neighborhood structure. As shown in Table~\ref{tab:app-reverse_dice_performance}, first-order spatial models degrade rapidly: at the $20\%$ perturbation ratio, GCN and GAT drop by $23.82\%$ and $21.65\%$, respectively. Polynomial spectral baselines are more stable, with BernNet and JacobiConv decreasing by $7.90\%$ and $7.67\%$. DCQ-GNN obtains the smallest degradation at $20\%$, retaining $85.63\%$ accuracy with a relative drop of $4.45\%$. At $5\%$, DCQ-GNN ties JacobiConv in rounded accuracy ($88.37\%$), while at $10\%$ and $20\%$ it achieves the best accuracy among all compared methods. These results support the robustness of curvature-controlled quadratic filtering under adversarial topology modification in homophilic graphs.

\subsection{Spectral Visualization (RQ5)}
\begin{figure}[h]
	\centering
	\includegraphics[width=0.97\linewidth]{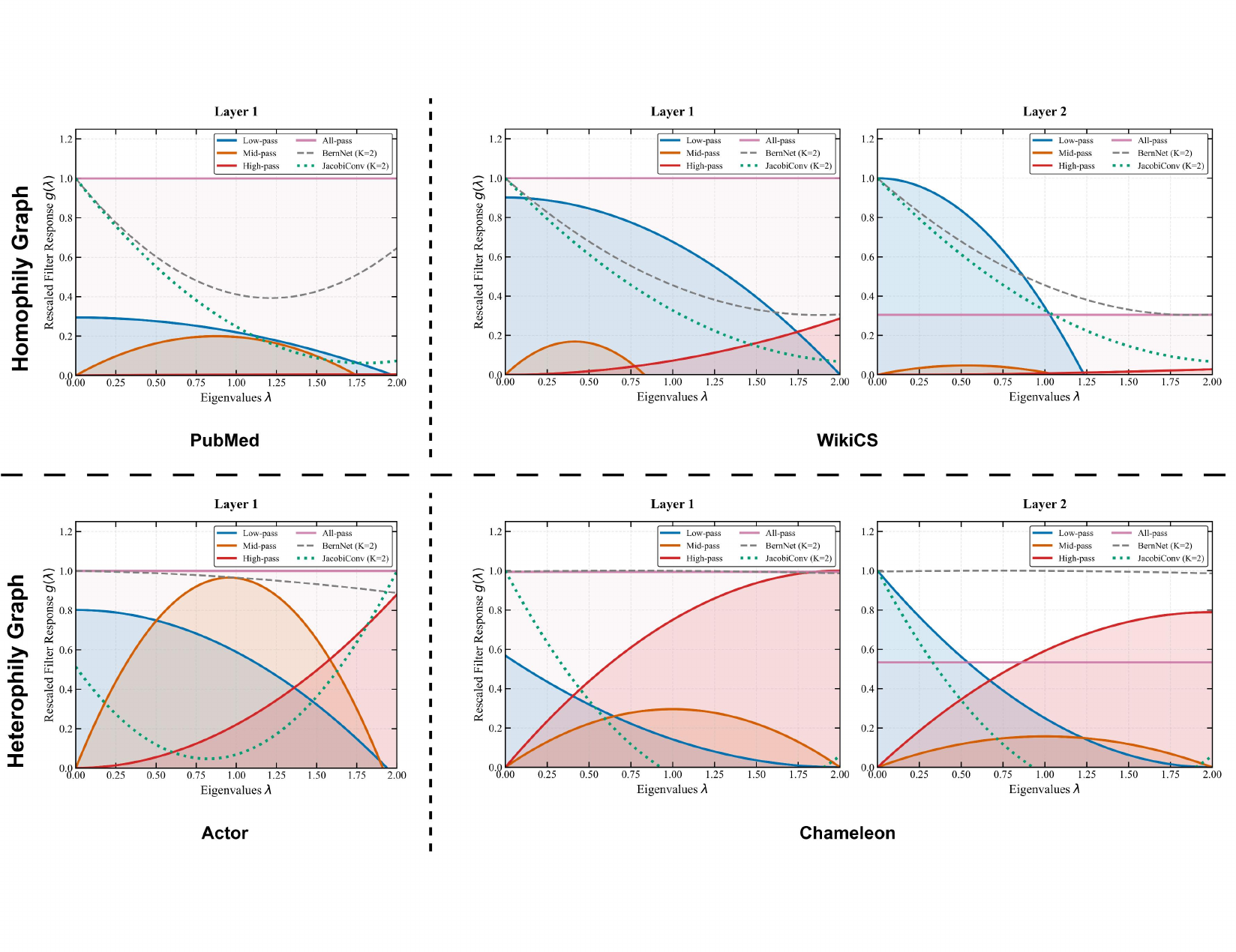}
	\caption{Learned layer-wise spectral responses of DCQ-GNN compared with second-order ($K=2$) polynomial baselines on homophilic (PubMed, WikiCS) and heterophilic (Actor, Chameleon) graphs. Shaded regions denote the explicit frequency channels learned by DCQ-GNN.}
	\label{fig:app-filter_profiles}
\end{figure}

To further analyze the spectral behavior of the proposed model across varying homophily regimes, we visualize the learned spectral responses of the DCQ-GNN filter bank alongside established polynomial baselines in Fig.~\ref{fig:app-filter_profiles}. To ensure a fair comparison in terms of model complexity, BernNet and JacobiConv are restricted to second-order polynomial capacity ($K=2$), matching the quadratic formulation of DCQ-GNN.

As shown in Fig.~\ref{fig:app-filter_profiles}, DCQ-GNN exhibits topology-dependent spectral adaptation. In homophilic graphs (PubMed and WikiCS), the learned response is dominated by the convex low-pass component, emphasizing low-frequency smoothing, while the high-pass channel displays limited activation. In contrast, for heterophilic graphs (Actor and Chameleon), higher-frequency components become more pronounced, indicating increased reliance on disassortative structural signals.

The visualization also reveals the behavior of standard polynomial approximations under low-order constraints. On the heterophilic Chameleon dataset, JacobiConv converges to a monotonic low-pass profile under the $K=2$ setting, limiting its ability to emphasize higher-frequency bands. Similarly, BernNet learns an approximately flat response on the Actor dataset. While such a flat profile is consistent with empirical observations that structure-agnostic MLP-style models can perform competitively on certain heterophilic graphs, it indicates limited spectral differentiation at low polynomial order.

In contrast, DCQ-GNN decomposes the spectrum into complementary frequency channels rather than fitting a single global polynomial. This channel-wise decomposition enables the model to combine frequency-specific responses through node-adaptive gating, facilitating flexible spectral discrimination without increasing polynomial degree.

\subsection{Node Classification with Dense Split (Balanced Classes, RQ6)}
\begin{table*}[ht] 
	\centering
	\caption{Node classification accuracy (mean $\pm$ standard deviation) on four benchmark datasets under the dense split setting. The best result is highlighted in \textbf{bold} and the second best is \underline{underlined}.}
	\resizebox{\textwidth}{!}{ 
		\begin{tabular}{lcccccc} 
			\toprule
			\textbf{Datasets}
			& \textbf{Photo} &  \textbf{Computers} & \textbf{Chameleon} & \textbf{Squirrel} & \textbf{Avg. Rank} \\ 
			\textbf{Homophily (Node)} & $h=0.83$ & $h=0.78$ & $h=0.10$ & $h=0.08$ & {} \\ 
			\midrule 
			GCN & $93.86 \pm 0.56$ & $88.80 \pm 0.27$ & $44.96 \pm 2.95$ & $32.14 \pm 0.94$ & 7.50 \\
			GCNII & $\mathbf{95.74 \pm 0.38}$ & $\underline{90.48 \pm 0.27}$ & $54.07 \pm 1.20$ & $37.34 \pm 1.80$ & 3.75 \\
			GAT & $94.80 \pm 0.44$ & $90.25 \pm 0.79$ & $45.82 \pm 2.82$ & $33.87 \pm 0.73$ & 6.00 \\
			BernNet & $93.63 \pm 0.35$ & $87.64 \pm 0.44$ & $68.29 \pm 1.58$ & $51.35 \pm 0.73$ & 6.00 \\
			AutoSGNN & $95.12 \pm 0.37$ & $90.06 \pm 0.63$ & $66.30 \pm 1.94$ & $56.58 \pm 1.14$ & 4.50 \\
			LHK-GNN & $95.10 \pm 0.31$ & ${90.12 \pm 0.29}$ & $\underline{70.67 \pm 0.92}$ & $56.34 \pm 0.87$ & 4.00 \\
			\hline
			\textbf{DCQ-GNN (ours)} & $95.41 \pm 0.53$ & $\mathbf{90.95 \pm 0.65}$ & $68.16 \pm 1.39$ & $\underline{56.61 \pm 1.49}$ & $\bf \underline{2.50}$ \\
			\bottomrule
	\end{tabular}}
	\label{tab:app-dense_split_results}
\end{table*}
In the main paper (Table~\ref{tab:main_results}), models are evaluated under SRS splits without enforcing class balance. In contrast, Table~\ref{tab:app-dense_split_results} reports results under the dense split protocol.

To provide a complementary evaluation and enable comparison with recent studies, we additionally evaluate DCQ-GNN under a class-balanced ``dense split'' protocol. In this setting, the training set allocates an equal number of nodes to each class (specifically, $N_{train} = N \times train\_ratio / C$ nodes per class). Although this protocol modifies the inherent class imbalance, it is widely adopted in contemporary spectral GNN literature.

For completeness, we include two recent baselines in this evaluation: AutoSGNN and LHK-GNN. As reported in Table~\ref{tab:app-dense_split_results}, DCQ-GNN achieves the best overall average rank (2.50) across the evaluated datasets and outperforms AutoSGNN and LHK-GNN on average.  Under this protocol, DCQ-GNN achieves competitive performance across all datasets and obtains the best average rank. Overall, these results indicate that the proposed convex--concave quadratic filter bank maintains competitive performance under both naturally imbalanced and class-balanced training settings.

\section{Related Work}\label{app:related}
We review prior work most relevant to DCQ-GNN from three perspectives: spectral graph convolution, adaptive polynomial spectral filtering, and filter-bank mechanisms for learning on heterophilic graphs.

\subsection{Spectral Graph Convolutional Networks}

Spectral graph neural networks originate from graph signal processing, which extends classical signal filtering to irregular graph domains \cite{shuman2013emerging}.
Early formulations defined convolution in the graph Laplacian eigenbasis \cite{bruna2013spectral}, but their reliance on eigendecomposition limited scalability.
ChebyNet \cite{defferrard2016convolutional} alleviates this issue by approximating spectral filters using Chebyshev polynomials, enabling localized and efficient propagation without explicit eigenvectors.
GCN \cite{kipf2017semi} further simplifies this formulation to a first-order polynomial approximation, resulting in an efficient and widely adopted low-pass operator.

Many spatial GNN architectures can be interpreted within this spectral framework and exhibit an inherent low-frequency bias, including GraphSAGE \cite{hamilton2017inductive}, GAT \cite{velivckovic2018graph}, JK-Net \cite{xu2018representation}, and GCNII \cite{chen2020simple}.  While such low-pass smoothing behavior is well suited to homophilic graphs, it is often suboptimal in heterophilic settings, where adjacent nodes tend to have dissimilar labels and discriminative information may reside in higher-frequency components \cite{balcilar2021analyzing,zhu2020beyond}.

\subsection{Adaptive Polynomial Spectral Filters}

To overcome the limitations of fixed low-pass filtering, several approaches have explored adaptive polynomial spectral responses.
GPR-GNN \cite{chien2021adaptive} learns propagation coefficients inspired by generalized PageRank \cite{brin1998anatomy}, allowing data-dependent aggregation across multiple hop distances. BernNet \cite{he2021bernnet} adopts Bernstein polynomial bases with constrained coefficients to improve numerical stability while supporting diverse spectral profiles.  JacobiConv \cite{wang2022powerful} employs orthogonal Jacobi polynomials to mitigate instability in higher-order expansions, and UniFilter \cite{huang2024universal} further enhances adaptivity by dynamically constructing spectral bases tailored to graph heterophily.

Although these methods significantly increase spectral flexibility, achieving sharper frequency selectivity typically requires higher polynomial orders.
This may enlarge the hypothesis space, increase sensitivity to noise, and incur additional computational cost due to repeated propagation steps \cite{levie2018cayleynets}.
In contrast, our work investigates whether a strictly second-order design can capture much of the practical benefit of spectral adaptivity, while offering improved stability and parameter efficiency.

\subsection{Filter Banks for Heterophilic Graphs}

A complementary line of research addresses heterophily by decomposing graph signals into multiple frequency components and combining them through learned mixing mechanisms.
FAGCN \cite{bo2021beyond} integrates low- and high-frequency information via frequency-adaptive gating, while related multi-channel formulations have been explored in AutoGCN \cite{wu2022beyond} and Adaptive Channel Mixing (ACM) \cite{luan2022revisiting}.
More recently, Zheng et al.~\cite{zheng2021framelets} introduced spectral tight framelet filter banks composed of one low-pass and multiple high-pass filters, enabling principled frequency separation with stable reconstruction guarantees.

Despite their effectiveness, most existing filter-bank approaches rely on linear filters as their fundamental building blocks. From a spectral viewpoint, repeated linear low-pass propagation exhibits no intrinsic curvature and therefore provides limited capability for suppressing structurally induced interference. Consequently, perturbations or corrupted components may be repeatedly propagated and even amplified \cite{huang2023robust,zugner2018adversarial}. In contrast, DCQ-GNN employs quadratic filters with an explicit convex--concave decomposition, thereby introducing controllable curvature into the spectral response. As shown in our theoretical analysis, this curvature enables selective attenuation of interference-prone frequency components while preserving informative signal bands. Consequently, DCQ-GNN realizes a structured filter-bank architecture that yields sharper spectral discrimination and improved robustness, without relying on high-order polynomial expansions or sacrificing the computational efficiency of low-order message passing.

Several studies have explored multi-channel filtering to address graph heterophily. Although Graph Framelets provide powerful multi-scale frequency decompositions, they typically require complex frame construction and high-order spectral expansions. In contrast, DCQ-GNN provides a lightweight quadratic alternative: by exploiting curvature polarity rather than multi-level scaling, the proposed model achieves comparable spectral separation with a strictly two-hop computational footprint.

\subsection{Connections to Continuous Diffusion and Graph Transformers}
\paragraph{Continuous Graph Diffusion.}
	GRAND~\cite{chamberlain2021grand} and GraphCON~\cite{rusch2022graph} define node
	dynamics via neural ODEs on graphs, with stability certified through Lyapunov
	analysis of the continuous flow. The second-order filter $g(\lambda)=a_0+a_1\lambda
	+a_2\lambda^2$ admits a discrete-time interpretation as a two-step Runge--Kutta
	approximation of such diffusion: $a_1$ governs the diffusion rate, while $a_2$
	introduces a curvature correction analogous to the momentum coupling in GraphCON.
	Theorem~\ref{app:thm-strucprtur} can thus be viewed as the spectral counterpart of those
	Lyapunov certificates. Both characterize bounded output variation under bounded
	structural perturbation, whereas our result operates on filter coefficients rather
	than ODE parameters.
	\paragraph{Graph Transformers.}
	Transformer-based architectures such as Graphormer~\cite{ying2021transformers} and
	GPS~\cite{rampavsek2022recipe} attend globally over node pairs, achieving strong
	performance at the cost of $\mathcal{O}(N^2)$ attention complexity or explicit
	structural positional encodings. DCQ-GNN targets a complementary regime in which
	$\mathcal{O}(M)$ polynomial propagation is preferred for efficiency and spectral
	interpretability. We therefore view Graph Transformers as a distinct design family
	rather than direct competitors.

\end{document}